\documentclass[11pt]{article}

\usepackage[a4paper,margin=1in]{geometry}
\usepackage{amsmath,amsfonts,amssymb,amsthm}
\usepackage{graphicx}
\usepackage{url}
\usepackage{hyperref}
\usepackage{booktabs}
\usepackage{algorithm}
\usepackage{algorithmic}
\usepackage{listings}
\usepackage{xcolor}
\usepackage{natbib}
\usepackage{array}
\usepackage{mathtools}
\usepackage{tabularx}
\usepackage{enumerate}
\usepackage[inline,shortlabels]{enumitem}
\usepackage{subcaption}
\usepackage{multirow}
\usepackage{bm}
\usepackage{hyperref}
\usepackage{mmll}
\newcommand{\dd}{\mathrm{d}}
\usepackage[capitalize]{cleveref}

\usepackage{arydshln}
\usepackage{wrapfig}
\usepackage{color}
\usepackage{algorithm}
\usepackage{algorithmic}
\usepackage{xifthen}
\usepackage{etoc}
\usepackage{pifont}
\newcommand{\cmark}{\textcolor{green!55!black}{\ding{51}}}
\newcommand{\xmark}{\textcolor{red!75!black}{\ding{55}}}
\makeatletter

\renewcommand{\maketag@@@}[1]{\hbox{\m@th\normalsize\normalfont#1}}%

\makeatother

\title{Heavy-Tailed Flow Matching via Random Clocks}

\newcommand{\authoremail}[1]{%
  {\small\normalfont\href{mailto:#1}{\nolinkurl{#1}}}%
}

\newcommand{\authorblock}[3]{%
  \begin{tabular}[t]{@{}c@{}}
    #1\textsuperscript{#2}\\[-1mm]
    \authoremail{#3}
  \end{tabular}%
}

\author{%
  \begin{tabular}{@{}ccc@{}}
    \authorblock
      {Zhouhao Yang}
      {*}
      {zyang145@jhu.edu}
    &
    \authorblock
      {Yezhen Wang}
      {\ensuremath{\dagger}}
      {e1154541@u.nus.edu}
    &
    \authorblock
      {Kenji Kawaguchi}
      {\ensuremath{\dagger}}
      {kenji@comp.nus.edu.sg}
    \\[1.2em]
    \multicolumn{3}{c}{%
      \begin{tabular}{@{}cc@{}}
        \authorblock
          {Vladimir Braverman}
          {\ensuremath{\ddagger}}
          {vova@cs.jhu.edu}
        &
        \hspace{2em}
        \authorblock
          {Haoyang Cao}
          {*}
          {hycao@jhu.edu}
      \end{tabular}%
    }
  \end{tabular}%
}

\date{}

\newcommand{\Law}{\operatorname{Law}}

\begin{document}

\begingroup
\renewcommand{\thefootnote}{\fnsymbol{footnote}}

\maketitle

\footnotetext[1]{%
Department of Applied Mathematics and Statistics,
Johns Hopkins University.}

\footnotetext[2]{%
School of Computing,
National University of Singapore.}

\footnotetext[3]{%
Department of Computer Science,
Johns Hopkins University.}

\endgroup

\setcounter{footnote}{0}

\maketitle

\begin{abstract}
Heavy-tailed data arise in many domains where rare events carry disproportionate importance, such as imbalanced image datasets, financial returns, and weather extremes. Standard diffusion and flow-matching models typically begin from Gaussian noise or Gaussian source distributions, which yield tractable training targets but provide a poor inductive match for heavy-tailed data. We propose Heavy-Tailed Flow Matching via Random Clocks (HTFM), a framework that portrays heavy-tailed sources as mixtures of clock-conditioned Gaussian sources. Conditioning on a given clock path, the source distribution and flow are  Gaussian; marginalizing over the clock gives a Gaussian scale mixture covering Gaussian, $\alpha$-stable, and Student-t families. To make the clock-conditioned vector field practical, we encode the path-valued clock using truncated logsignature features, allowing the velocity field to adapt to the realized conditional space with negligible overhead. Empirically, on 2D imbalanced $\alpha$-stable mixtures, CIFAR10-LT, and HRRR weather fields, HTFM improves mode coverage, sample quality, and tail-statistic recovery over Gaussian flow matching and competitive heavy-tailed baselines, while retaining the low-NFE sampling advantage of flow matching. Moreover, the random-clock formulation further provides a practical tail-control interface: by varying only the clock law or tail parameter, the same architecture can calibrate the ``heaviness'' of generated tails across different distribution families.
\end{abstract}

\section{Introduction}
\label{sec:introduction}

In many real-world applications, rare or atypical observations carry disproportionately large impact relative to typical ones~\citep{coles2001introduction,resnick2007heavy}: hurricanes and heatwaves in operational weather and climate models~\citep{seneviratne2021weather,dowell2022high,pathak2022fourcastnet,grundemann2022rarest}, the tail of asset returns in financial risk modeling~\citep{cont2001empirical,embrechts2013modelling,cont2026tail}, and minority classes in long-tailed image datasets~\citep{cao2019learning,liu2019large,zhang2024long}. A generative model is useful in such regimes only insofar as it reproduces these rare events with the right frequency and magnitude.

This requirement is difficult for two related reasons. First, extreme observations are rare 
by construction, so empirical training data contains few samples from the relevant tail 
regions~\citep{grundemann2022rarest}. Second, many deep generative modeling pipelines begin 
from simple light-tailed randomness: variational autoencoders and generative adversarial 
networks commonly use Gaussian or uniform latent priors~\citep{kingma2014auto, goodfellow2014generative}, 
while state-of-the-art diffusion and flow-matching models typically use Gaussian noising processes 
or Gaussian source distributions~\citep{ho2020denoising, song2020score, lipman2022flow}. 
When the data distribution has substantially heavier tails, this light-tailed source is a poor 
inductive match: the model must learn tail behavior through the learned transport, which can 
make rare  events difficult to reproduce accurately.

A natural response is to replace the Gaussian source or noise with a heavy-tailed alternative. In diffusion and score-based modeling, L\'evy-It\^{o} models~\citep{yoon2023score,popovimproved} use \(\alpha\)-stable noise but require nonlocal fractional-score machinery and are tied to a particular stable family. Denoising L\'evy probabilistic models (DLPM)~\citep{shariatianheavy} adapt the DDPM~\citep{ho2020denoising} construction to \(\alpha\)-stable noise and use a denoising objective with robust aggregation over heavy-tailed randomness. Student-\(t\)-based diffusion and flow models~\citep{pandey2024heavy} use Student-\(t\) kernels with a \(\gamma\)-divergence training loss and are restricted to the finite-variance regime \(\nu>2\). Other constructions specialize the source or noise to heavy-tailed families such as Cauchy noise~\citep{lian2025cauchy} or finite-activity jump diffusions~\citep{baule2025generative}. Related tail limitations have also been studied for normalizing flows~\citep{jaini2020tails,hickling2025flexible}. These approaches demonstrate the value of heavy-tailed generative mechanisms, but they are typically built around a specific tail family and the corresponding score, kernel, sampler, or loss aggregation. In practice, the appropriate tail mechanism is rarely known in advance; a method built for one heavy-tailed family may be misspecified for another.

This leaves a gap for flow matching~\citep{lipman2022flow,liu2022flow,albergo2022building}. Flow matching learns an ordinary differential equation (ODE) that transports a source distribution to the data distribution by regressing tractable velocity targets along prescribed probability paths, rather than simulating a long reverse diffusion chain. This often enables high-quality sampling with a substantially lower number of function evaluations (NFE), and rectified-flow and optimal-transport variants further sharpen this sample-efficiency advantage~\citep{liu2022rectified,tong2023improving,pooladian2023multisample,kornilov2024optimal}. A heavy-tailed analogue of flow matching should therefore preserve the low-NFE advantages of flow matching while allowing source laws with heavier tails than Gaussians. Existing heavy-tailed diffusion and flow constructions make important progress, but their targets or divergences are usually derived separately for each chosen heavy-tailed family.

In this paper, we propose \emph{Heavy-Tailed Flow Matching via Random Clock} (HTFM). A random clock is a nondecreasing path-valued latent variable that controls the conditional covariance of the source distribution. Conditioning on a given clock path, the source distribution and the corresponding affine flow between source and data are Gaussian, so the endpoint-conditioned velocity target has the same algebraic form as in Gaussian conditional flow matching. After marginalizing over the clock, the source and flow can be heavy-tailed.

The random-clock viewpoint gives a unified construction for several representative tail mechanisms. A deterministic clock recovers vanilla Gaussian flow matching; an \(\alpha\)-stable subordinator clock yields \(\alpha\)-stable source marginals; and an inverse-gamma clock recovers Student-\(t\) marginals for any \(\nu>0\), including the infinite-variance regime \(\nu\le2\) not covered by finite-variance Student-\(t\) constructions~\citep{pandey2024heavy}. Across these choices, the endpoint-conditioned target is derived in the same clock-conditioned Gaussian space, avoiding family-specific score formulas, direct use of generally unavailable \(\alpha\)-stable densities, or Student-\(t\)-specific kernels. 


The remaining difficulty is that the pathwise optimal vector field is determined by the realized clock path, an infinite-dimensional object. Passing the full path into the velocity field would amount to an operator-learning problem~\citep{kovachki2023neural} and would undermine the practical simplicity of flow sampling. We instead characterize the clock path by a finite-dimensional truncated logsignature feature~\citep{lyons2014rough,chevyrev2025primer}. Path signatures enable universal finite-dimensional representations of stream-valued data and have been used as kernels, feature maps, and differentiable neural components~\citep{levin2013learning,kiraly2019kernels,kidger2019deep,kidger2021signatory,morrill2021neural}. This lets the velocity field adapt to the realized clock at negligible training and sampling cost.

Empirically, our experiments support three findings about the random clock: (i) heavy-tailed clocks improve over the Gaussian-clock variant across benchmarks; (ii) different clock families behave differently even at matched polynomial tail-decay rates, indicating that the clock family is a meaningful design knob beyond tail heaviness; and (iii) explicit clock conditioning of the velocity field provides gains beyond simply sampling from a heavy-tailed source.

The contributions of our paper can be summarized as follows:
\begin{enumerate*}
    \item \textbf{A random-clock framework for heavy-tailed flow matching.}
We introduce HTFM, a flow-matching framework in which a path-valued random
clock controls the conditional source covariance. Conditioning on the realized
clock gives a Gaussian conditional source and a tractable endpoint-conditioned
affine target; marginalizing over clocks recovers Gaussian scale-mixture
families, including Gaussian, \(\alpha\)-stable, and Student-\(t\) sources.
    \item \textbf{Signature-based clock conditioning.} We encode the path-valued clock using truncated logsignature features, allowing
the neural velocity field to adapt to the realized conditional Gaussian space
with negligible additional cost. This provides a finite-dimensional interface
for clock-conditioned vector fields.
    \item \textbf{Empirical results.} Experiments on 2D imbalanced \(\alpha\)-stable mixtures, CIFAR10-LT, and HRRR
    weather fields show that heavy-tailed clock laws improve over the Gaussian-clock
variant and that the clock family remains important even at matched polynomial
tail-decay rates. On CIFAR10-LT, HTFM compares favorably with existing
heavy-tailed diffusion and ODE-based baselines, especially at low NFE; ablations
further show that signature conditioning improves over using a heavy-tailed
source alone.
\end{enumerate*}

\section{Preliminaries}
\label{sec:prelim}

\subsection{Flow matching}
\label{subsec:gaussian-fm}

Flow matching learns a time-dependent vector field whose ODE transports a
source distribution \(p_0\) to the data distribution \(p_{\mathrm{data}}\).
Let \(X_t\in\RR^d\) denote the state of the flow at time \(t\in[0,1]\),
a stochastic process with \(X_0\sim p_0\) and marginal law
\(p_t=\Law(X_t)\). Concretely, consider the ODE flow
\begin{equation}
\frac{\dd X_t}{\dd t}=u_t(X_t),
\qquad t\in[0,1],
\qquad X_0\sim p_0,
\label{eq:cnf-ode}
\end{equation}
whose marginal density evolves according to the continuity equation. The goal is to choose \(u_t\) so that
\(p_1=p_{\mathrm{data}}\). Given the probability flow
\((p_t)_{t\in[0,1]}\) and its marginal velocity field \(u_t\), the ideal
marginal flow-matching (FM) objective is
\begin{equation}
\cL_{\mathrm{FM}}(\theta)
=
\EE_{t\sim\mathrm{Unif}(0,1)}
\EE_{X_t\sim p_t}\Big[
\big\|u_\theta(X_t,t)-u_t(X_t)\big\|_2^2
\Big].
\label{eq:general-fm-loss}
\end{equation}
However, the marginal field \(u_t\) is usually intractable. Conditional flow matching (CFM)
therefore conditions on a data endpoint \(X_1\sim p_{\mathrm{data}}\) and
specifies an endpoint-conditioned path \(p_t(\cdot\mid X_1)\) with tractable
conditional velocity \(v_t(\cdot\mid X_1)\). The corresponding objective is
\begin{equation}
\cL_{\mathrm{CFM}}(\theta)
=
\EE_{t\sim\mathrm{Unif}(0,1)}
\EE_{X_1\sim p_{\mathrm{data}}}\,
\EE_{X_t\sim p_t(\cdot\mid X_1)}\Big[
\big\|u_\theta(X_t,t)-v_t(X_t\mid X_1)\big\|_2^2
\Big].
\label{eq:general-cfm-loss}
\end{equation}
Under standard $L^2$-integrability and differentiability assumptions, the
conditional and marginal objectives differ only by a
\(\theta\)-independent constant, so their gradients agree
\citep{lipman2022flow,tong2023improving}. This \(L^2\)-projection principle is
the basic mechanism we preserve in our  framework proposed later in Section~\ref{sec:HTFM}.

\subsection{Heavy-tailed noises as Gaussian scale mixtures}
\label{subsec:ht-noise}

Gaussian and heavy-tailed laws are primarily distinguished by their tail decay.
Gaussian tails are exponentially small, whereas many heavy-tailed laws exhibit
only polynomial decay. Such laws
are often better suited to modeling rare but non-negligible large deviations.

A convenient latent-Gaussian framework for symmetric heavy-tailed modeling is
provided by Gaussian scale mixtures. Let \(V\) be a latent scale variable and
let \(G\sim\cN(0,I_d)\) be independent of \(V\). A random vector
\(X\in\RR^d\) is called a Gaussian scale mixture if it admits a representation
of the form \(X=L(V)G\), where \(L(V)\in\RR^{d\times d}\) is a measurable
random matrix. Equivalently,
\begin{equation}
X\mid V
\sim
\cN\bigl(0,\Sigma(V)\bigr),
\qquad
\Sigma(V):=L(V)L(V)^\top .
\label{eq:gsm-conditional}
\end{equation}
Thus, conditioning on the latent variable \(V\), the law remains Gaussian,
whereas marginalizing over specific choices of \(V\) can recover a heavy-tailed marginal law. The
mechanism is simple: rare large values of the latent scale generate broad
conditional Gaussians, and the mixture assigns non-negligible probability to
large deviations even though every conditional component is light-tailed.
Representative examples include symmetric \(\alpha\)-stable laws and
Student-\(t\) laws; their concrete scale-mixture forms are recalled in
Appendix~\ref{appendix:gsm-examples}.

This latent-Gaussian viewpoint suggests a natural strategy for flow matching.
Rather than trying to regress directly against a heavy-tailed marginal flow, we
condition on the latent mixing object so that the interpolation becomes
conditionally Gaussian. We may then derive the conditional flow-matching target
in this Gaussian regime and exploit the \(L^2\) projection principle from
Section~\ref{subsec:gaussian-fm}. This conditional-Gaussian perspective is also
closely related to \citep{shariatianheavy}, which uses a conditioned representation of
\(\alpha\)-stable noise to build tractable denoising targets for diffusion models.

The Gaussian scale-mixture viewpoint is already broad enough to cover many
symmetric heavy-tailed laws and will be the focus of the main development in
this paper. A broader latent-Gaussian framework, obtained by allowing a
nonzero conditional mean in \eqref{eq:gsm-conditional} and hence passing from
Gaussian scale mixtures to Gaussian mean-variance mixtures, can also
accommodate more general families of asymmetric heavy-tailed laws; we defer this extension to
Appendix~\ref{appendix:gmvm}.

\section{Heavy-Tailed Flow Matching via Random Clocks}
\label{sec:HTFM}
This section develops the random-clock framework for heavy-tailed flow matching; a related diffusion-model extension is given in Appendix~\ref{appendix:diffusion-extension}.

The latent-Gaussian representation in Section~\ref{subsec:ht-noise} suggests
replacing the scalar mixing variable \(V\) by a path-valued latent object, if we want to work on flow matching. We
call this object a \textbf{random clock}. Formally, let
\(
\mathbb D([0,1];\RR_+)
:=
\{\tau:[0,1]\to\RR_+ \;\big|\; \tau
\text{ is c\`adl\`ag and nondecreasing}\},
\)
equipped with its Borel \(\sigma\)-field. A random clock is a random element
on a probability space \((\Omega_T,\cF_T,\PP_T)\),
\(
T:\Omega_T\to \mathbb D([0,1];\RR_+),
\)
and we write \(T_{[0,1]}=(T_t)_{t\in[0,1]}\) for the full realized clock path.
Throughout, \(T_0=0\) almost surely.

The clock is sampled independently of the base randomness used to construct the
source and target variables, such as a standard Gaussian \(G\) and
\(X_1\sim p_{\mathrm{data}}\). Clock-dependent sources are nevertheless
allowed: for example, \(X_0=L(T_{[0,1]})G\) is conditionally Gaussian given the
realized clock path but is not independent of the clock after it is formed.

\begin{figure}[t]
    \centering
    \includegraphics[width=0.92\linewidth]{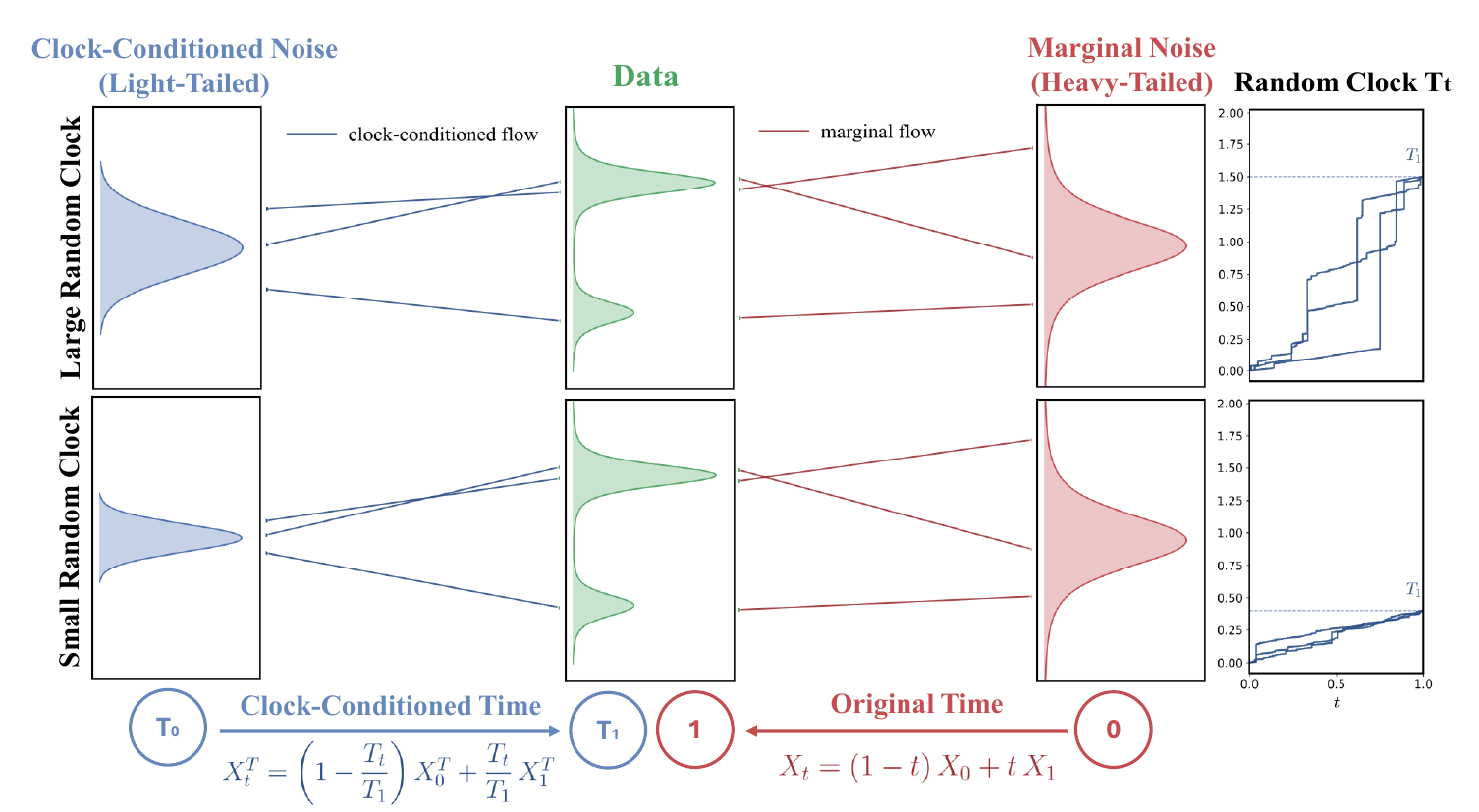}
    \vspace{-0.5em}
    \caption{Overview of random-clock conditioning for heavy-tailed flow matching.
The right column shows two realized clock paths: a large clock realization
(top) and a small clock realization (bottom). For each fixed clock path $T$,
the source distribution is conditionally Gaussian, and the model learns
a clock-conditioned flow from this light-tailed conditional source to the data
distribution. Larger clock realizations induce broader conditional source
distributions, while smaller clock realizations induce more concentrated ones.
After marginalizing over clock paths, the mixture of these conditional Gaussian
sources yields a heavy-tailed marginal source distribution. }
    \vspace{-1em}
    \label{fig:toy_illustration}
\end{figure}

\subsection{Clock-conditioned dynamics and objectives}
\label{subsec:clock-dynamics-loss}

Motivated by Section~\ref{subsec:ht-noise}, we formulate heavy-tailed flow
matching \textbf{pathwise}, that is, conditioning on a realized clock path,
and reserve the outer averaging over the random clock for the final
objective.

Let \(T=(T_t)_{t\in[0,1]}\) be a random clock\footnote{When appearing as a
superscript, \(T\) indicates the clock path and is not an exponent or transpose.}, and fix a realization of the full
clock path \(T_{[0,1]}\). We write \(X_t^T\), \(p_t^T\), and \(u_t^T\) for the
clock-conditioned process, its law \(p_t^T:=\Law(X_t^T)\), and its velocity
field under this realized clock path. Due to the independence of the data law and the random clock, we write \(X_1^T:=X_1\) for the
terminal data variable, and \(p_{\mathrm{data}}(\cdot\mid T_{[0,1]})=p_{\mathrm{data}}\).  We consider
the clock-conditioned ODE
\begin{equation}
\frac{\dd X_t^T}{\dd t} = u_t^T(X_t^T),
\qquad
t\in[0,1],
\qquad
X_0^T \sim p_0^T := p_0(\cdot \mid T_{[0,1]}),
\label{eq:clock-conditioned-ode}
\end{equation}
whose conditional density satisfies the continuity equation
\(
\partial_t p_t^T + \nabla\cdot(p_t^T  u_t^T) = 0.
\)
Equivalently, one may view \(u_t^T(\cdot)\) as the evaluation at
\(T_{[0,1]}\) of a measurable map
\(u_t:\RR^d\times \mathbb D([0,1];\RR_+)\to\RR^d\). Our goal is
to construct, for each realized \(T_{[0,1]}\), a flow from \(p_0^T\) at
time \(t=0\) to the terminal law \(p_{\mathrm{data}}\) at time \(t=1\).

Intuitively, conditioning on a realized clock path locally stretches or
compresses physical time. This conditioning moves the heavy-tailedness
into the outer randomness of the clock, while the inner pathwise problem
becomes conditionally Gaussian and therefore amenable to the usual
\(L^2\)-projection argument of flow matching. An illustration is provided
in Figure~\ref{fig:toy_illustration}.
\vskip-15pt
\paragraph{Training objective.}
For brevity, after fixing \(T_{[0,1]}\) we write \(\EE_{X_t^T}\) for
expectation under \(p_t^T\), and \(\EE_{X_t^T\mid X_1^T}\) for
expectation under \(\Law(X_t^T\mid X_1^T,T_{[0,1]})\). The pathwise
marginal quadratic risk and the corresponding heavy-tailed
flow-matching objective are
\begin{equation}
\mathcal L_{\mathrm{HTFM}}(\theta)
:=
\EE_{t\sim\mathrm{Unif}(0,1)}
\EE_T\underbrace{\Bigl[
\EE_{X_t^T}\bigl[
\|u_\theta(X_t^T,t,T_{[0,1]})-u_t^T(X_t^T)\|_2^2
\bigr]
\Bigr]}_{=:\,R_t^{\mathrm{marg}}(\theta;T)}.
\label{eq:htfm-loss}
\end{equation}
The risk \(R_t^{\mathrm{marg}}(\theta;T)\) is typically intractable
because the pathwise marginal velocity \(u_t^T\) is unknown. As in
\cite{lipman2022flow,liu2022flow}, we therefore introduce a tractable
endpoint-conditioned formulation: for each fixed realized clock path
\(T_{[0,1]}\), denote a conditional path
\(p_t^T(\cdot\mid x_1):=\Law(X_t^T\mid X_1^T=x_1)\) and a corresponding
pathwise conditional velocity field
\(v_t^T(\cdot\mid x_1):\RR^d\to\RR^d\) such that, for a.e.\
\(t\in[0,1]\),
\begin{equation}
\small
p_t^T(x)
=
\int_{\RR^d} p_t^T(x\mid x_1)\,p_{\mathrm{data}}(\dd x_1),
\quad
p_t^T(x)\,u_t^T(x)
=
\int_{\RR^d} p_t^T(x\mid x_1)\,v_t^T(x\mid x_1)\,
p_{\mathrm{data}}(\dd x_1).
\label{eq:clock-compatibility-velocity}
\end{equation}
The corresponding endpoint-conditioned
heavy-tailed flow-matching objective writes
\begin{equation}
\mathcal L_{\mathrm{CHTFM}}(\theta)
:=
\EE_{t\sim\mathrm{Unif}(0,1)}
\EE_{T,X_1^T}\underbrace{\Bigl[
\EE_{X_t^T\mid X_1^T}\bigl[
\|u_\theta(X_t^T,t,T_{[0,1]})-v_t^T(X_t^T\mid X_1^T)\|_2^2
\bigr]
\Bigr]}_{=:\,R_t^{\mathrm{cond}}(\theta;X_1^T,T)}.
\label{eq:chtfm-loss}
\end{equation}

The following proposition is the clock-conditioned analogue of the standard conditional--marginal equivalence~\citep{lipman2022flow,liu2022flow}. The standard identity requires a finite-second-moment marginal velocity and breaks down for heavy-tailed sources; recasting it pathwise in~$T$ is what makes the equivalence available in our setting, as discussed after the proposition. Assumptions and proof are in Appendix~\ref{appendix:proof-chtfm-htfm-equivalence}.

\begin{proposition}[Pathwise conditional--marginal equivalence]
\label{prop:chtfm-htfm-equivalence}
Under the integrability and differentiability conditions in
Appendix~\ref{appendix:proof-chtfm-htfm-equivalence}, for a.e.\ \(t\) there
exists a nonnegative \(C_t(T)\), independent of \(\theta\), such that
\(
\EE\left[
R_t^{\mathrm{cond}}(\theta;X_1^T,T)
\,\middle|\,
T_{[0,1]}
\right]
=
R_t^{\mathrm{marg}}(\theta;T)+C_t(T).
\)
Consequently, whenever the risks and \(C_t(T)\) are integrable over
\((t,T)\) with \(t\sim\mathrm{Unif}(0,1)\),
\begin{equation}
\mathcal L_{\mathrm{CHTFM}}(\theta)
=
\mathcal L_{\mathrm{HTFM}}(\theta)
+
\EE_{t\sim\mathrm{Unif}(0,1)}\EE_T[C_t(T)].
\label{eq:population-risk-decomp}
\end{equation}
Thus, whenever the additive constant is finite, the two objectives have the
same minimizers; when differentiation may be exchanged with the expectations,
their gradients in \(\theta\) are also identical.
\end{proposition}
\vskip-10pt
\paragraph{Why a clock-conditioned formulation?}
Without this conditioning, one would attempt to learn the unconditional
marginal velocity \(u_t^\ast\) through an \(L^2\)-risk of the form
\(
\EE_{X_t}[\|u_\theta(X_t,t)-u_t^\ast(X_t)\|_2^2],
\)
which may fail to be finite under heavy tails, since heavy-tailed laws
need only admit moments of order \(p<2\). In practice, this infinite-variance
effect can appear as exploding or highly unstable quadratic training losses,
as observed in L\'evy-based models such as LIM and addressed explicitly by
DLPM through robust aggregation \citep{yoon2023score,shariatianheavy}. Replacing the \(L^2\)-loss
by a generic \(L^p\)-loss is not a satisfactory theoretical remedy either,
because it destroys the $L^2$ projection structure underlying the
conditional--marginal equivalence.
HTFM instead fixes the realized clock path before applying flow matching. In
this clock-conditioned space, the source is Gaussian and the standard
conditional-flow-matching projection identity applies pathwise under
Proposition~\ref{prop:chtfm-htfm-equivalence}. The heavy-tailedness is then
pushed to the outer clock average, whose finiteness is handled by separate
integrability assumptions.

\subsection{Signature feature for path conditioning}
\label{subsec:clock-signature-schedules}

The objective \eqref{eq:chtfm-loss} fixes the inner conditional regression
space through its dependence on the path-valued variable \(T_{[0,1]}\). Because
\(T_{[0,1]}\) is infinite-dimensional, it cannot be passed directly to a neural
field; we therefore summarize the realized clock path by a finite-dimensional
path-signature feature, taking the truncated logsignature \(\ell_m(T)\) defined
below. Path signatures originate in rough path theory
\citep{lyons2014rough,chevyrev2025primer} and have been widely used as
universal finite-dimensional encodings of stream-valued data in machine
learning, both as differentiable neural-network layers
\citep{kidger2019deep,kidger2021signatory} and as kernels and feature maps for
sequential data \citep{levin2013learning,kiraly2019kernels}; they are
computationally cheap to evaluate yet expressive, separating paths up to
tree-like equivalence on suitable compact families
(Proposition~\ref{prop:signature-linear-universality} in
Appendix~\ref{appendix:signature-linear-schedules}).
\vskip-15pt
\paragraph{Time-augmented signature.}
Let \(\mathbf T_s:=(s,T_s)\in\RR^2\), \(s\in[0,1]\), be the time-augmented
clock path. Its truncated signature of order \(m\) is the iterated-integral
tensor
\begin{equation}
{\small
\phi_m(T):=S^{(\le m)}(\mathbf T)
=
\left(
1,\,
\left(
\int_{0<u_1<\cdots<u_k<1}
\dd\mathbf T_{u_1}\otimes\cdots\otimes\dd\mathbf T_{u_k}
\right)_{k=1}^m
\right)
\in\RR^{D_m}.
}
\label{eq:clock-truncated-signature-feature}
\end{equation}
Time augmentation captures how clock mass is allocated in physical time.
For c\`adl\`ag clocks observed on a discretization
\(0=t_0<t_1<\cdots<t_N=1\), we linearly interpolate the sampled time-augmented
points \(\bigl((t_i,T_{t_i})\bigr)_{i=0}^{N}\) and take the truncated signature
of the resulting piecewise linear path; this is the operational meaning of
\(\phi_m(T)\) used throughout the paper. For a more compact representation we
additionally use the truncated logsignature
\(\ell_m(T):=\log_{\otimes}\bigl(\phi_m(T)\bigr)\), computed in the truncated
tensor algebra. The dimension reduction is substantial: for \(q=2\) and
\(m=3\), the raw signature has \(15\) coordinates whereas the nonconstant
logsignature has only \(5\). Precise jump conventions, dimension formulas, and
the lead--lag and piecewise-constant variants are deferred to
Appendix~\ref{appendix:log-sig}.
\vskip-15pt
\paragraph{Network conditioning.}
The neural velocity field uses \(\ell_m(T)\) as an auxiliary conditioning
input alongside the scalar time \(t\) and the state \(X_t^T\). In the
experiments we use the \emph{concat-std} variant: \(\ell_m(T)\) is
standardized using running statistics collected during training, projected to
a fixed embedding size by a small learnable linear map, and concatenated with
the time embedding before being injected into every residual block of the
velocity-field backbone. Replacing the full path argument in
\eqref{eq:chtfm-loss} by \(\ell_m(T)\) gives the practical
signature-conditioned endpoint objective
\begin{equation}
\mathcal L_{\mathrm{CHTFM}}^{\mathrm{sig}}(\theta)
:=
\EE_{t\sim\mathrm{Unif}(0,1)}
\EE_{T,X_1^T}\left[
\EE_{X_t^T\mid X_1^T}\left[
\bigl\|
u_\theta(X_t^T,t,\ell_m(T))-v_t^T(X_t^T\mid X_1^T)
\bigr\|_2^2
\right]
\right].
\label{eq:chtfm-loss-st}
\end{equation}
\vskip-15pt
\paragraph{Computational overhead.}
The signature pipeline adds a one-shot logsignature computation of complexity
\(\mathcal O(Nq^m)\) (with \(N\) clock observation points and \(q=2\),
\(m\le3\) in our experiments), plus a small linear projection of \(\ell_m(T)\)
concatenated with the time embedding. This cost is independent of the data
dimension and negligible compared to a single forward pass of the
velocity-field backbone on \(X_t^T\); no extra denoising or solver step is
introduced.

\subsection{Flow design and transport interpretation}
\label{subsec:flow-choices}

We now specify the affine flow family used in this paper. Let
\(
\Sigma: \mathbb D([0,1];\RR_+) \to \mathbb S_+^d
\)
be a measurable map into the cone of symmetric positive semidefinite
\(d\times d\) matrices. Suppose that, conditioned on the realized clock path,
\(
X_0^T \sim \cN\bigl(0,\Sigma(T_{[0,1]})\bigr),
\)
and that \(X_0^T\) and \(X_1^T\) are conditionally independent given
\(T_{[0,1]}\). Let \(\alpha_t^T\) and \(\beta_t^T\) be clock-dependent
schedules such that, for almost every realized clock path,
\(t\mapsto\alpha_t^T\) and \(t\mapsto\beta_t^T\) are absolutely continuous and
satisfy
\(
\alpha_0^T=0,
\
\beta_0^T=1,
\
\alpha_1^T=1,
\
\beta_1^T=0,
\
\beta_t^T>0\ \text{for }t\in[0,1).
\)
The clock-conditioned affine path is then
\begin{equation}
X_t^T = \alpha_t^T X_1^T + \beta_t^T X_0^T,
\qquad t\in[0,1].
\label{eq:affine-family}
\end{equation}
Under \eqref{eq:affine-family}, for each \(x_1\in\RR^d\) and each realized
clock path \(T_{[0,1]}\), the conditional law of \(X_t^T\) is Gaussian (derivation
in Appendix~\ref{appendix:proof-clock-affine-law}):
\begin{equation}
p_t^T(\cdot\mid x_1)
=
\Law(X_t^T\mid X_1^T=x_1)
=
\cN\bigl(\alpha_t^T x_1,\ (\beta_t^T)^2\Sigma(T_{[0,1]})\bigr).
\label{eq:clock-affine-law}
\end{equation}
The random clock thus influences the interpolation through the conditional
source covariance \(\Sigma(T_{[0,1]})\) and through the schedule pair
\((\alpha_t^T,\beta_t^T)\); in both cases the path remains conditionally
Gaussian once the clock is fixed. For a.e.\ \(t\in[0,1)\), differentiating
\eqref{eq:affine-family} at fixed clock realization and substituting
\(X_0^T=(X_t^T-\alpha_t^T X_1^T)/\beta_t^T\) yields the explicit
endpoint-conditioned target
\begin{equation}
v_t^T(x\mid x_1)
:=
\EE\left[
\frac{\dd X_t^T}{\dd t}
\,\middle|\,
X_t^T=x,\ X_1^T=x_1
\right]
=
\frac{\frac{\dd \beta_t^T}{\dd t}}{\beta_t^T}\,x
+\left(
\frac{\dd \alpha_t^T}{\dd t}
-
\frac{\frac{\dd \beta_t^T}{\dd t}}{\beta_t^T}\alpha_t^T
\right)x_1.
\label{eq:clock-affine-target}
\end{equation}
This is the same algebraic target as in Gaussian flow matching, with
the deterministic schedules replaced by their clock-dependent versions. Training
samples are generated by drawing \(T\) and then \(X_t^T\) from
\eqref{eq:affine-family}; the model regresses
\(u_\theta(X_t^T,t,\ell_m(T))\) toward \eqref{eq:clock-affine-target} in
\eqref{eq:chtfm-loss-st}. For generation, we sample \(T\), compute
\(\ell_m(T)\), draw \(X_0^T\sim\cN(0,\Sigma(T_{[0,1]}))\), and use a numerical
ODE solver on
\(
\dd X_t^T=u_{\theta^*}(X_t^T,t,\ell_m(T))\,\dd t.
\)

We now specialize to the straight-line flow used throughout our experiments; a more general signature-linear schedule family is deferred to Appendix~\ref{appendix:signature-linear-schedules}.
\vskip-15pt
\paragraph{Straight-line flow.}
Set \(\alpha_t^T=t\), \(\beta_t^T=1-t\). Then \eqref{eq:clock-affine-target}
reduces to
\(
v_t^T(x\mid x_1) = \tfrac{x_1-x}{1-t},
\ t\in[0,1),
\)
and the endpoint-conditioned target has no explicit clock dependence. The
clock still enters through the conditional law \eqref{eq:clock-affine-law} via
the source covariance, and hence through the projection defining the pathwise
marginal optimum
\(
u_t^T(x)=\EE[v_t^T(X_t^T\mid X_1^T)\mid X_t^T=x].
\)
A network that does not observe the clock can therefore only learn a clock-averaged field, whereas the clock-conditioned formulation learns the family of optima indexed by the realized clock path. This straight-line schedule cleanly recovers the heavy-tailed marginal source families used in our experiments, and additionally admits a pathwise $W_p$-geodesic interpretation under a $p$-optimal endpoint coupling that remains valid in the infinite-second-moment regime $p<2$. See Appendix~\ref{appendix:proof-clock-straight-geodesic} for the additional theorem.
\vskip-15pt
\paragraph{Heavy-tailed marginals from clock laws.}
With the straight-line schedule, choosing the clock law and source covariance
recovers the named heavy-tailed marginal families used in our experiments.
Throughout this paragraph, \(X_t\) denotes the clock-marginal variable obtained
by first sampling \(T\) and then sampling \(X_t^T\) in the corresponding
clock-conditioned space. Equivalently, 
\(\Law(X_t\mid X_1=x_1)
    =
    \EE_T \left[
        \Law(X_t^T\mid X_1^T=x_1,T_{[0,1]})
    \right].\)\footnote{
When \(\EE_T\) is applied to laws, it denotes the clock mixture:
\((\EE_T\mu^T)(B):=\int \mu^T(B)\,\PP_T(\dd T)\) for every Borel set
\(B\subseteq\RR^d\). This does not require finite moments of the
clock-marginal law.} 
For the representative examples, we use the area-clock
covariance \(\Sigma_A(T_{[0,1]}):=2A(T)I_d\), where
\(A(T):=\int_0^1 T_s\,\dd s\). The following theorem states that the same clock-conditioned
Gaussian construction can recover three commonly used source families: Gaussian, \(\alpha\)-stable, and Student-\(t\). The proof is deferred to
Appendix~\ref{appendix:proof-clock-choice-recovery}.

\begin{theorem}[Recovered clock marginals]
\label{thm:clock-choice-recovery}
Assume the affine path \eqref{eq:affine-family}, the clock-independent
straight-line schedule, and the covariance map \(\Sigma_A(T_{[0,1]})\). For notational brevity, write $p_t(\cdot\mid x_1) := \Law(X_t\mid X_1=x_1)$.
For every \(t\in[0,1)\) and \(x_1\in\RR^d\), the following identities hold. 

\noindent\textbf{(i)}
If \(T_t=t\), then
\(p_t(\cdot\mid x_1)=\cN(\alpha_t x_1,\beta_t^2I_d)\), with sub-Gaussian
tail decay.

\noindent\textbf{(ii)} Let \(\alpha\in(0,2)\), \(\rho=\alpha/2\), and let \(T\) be a \(\rho\)-stable
subordinator normalized by
\(\EE[e^{-sT_t}]=\exp(-(\rho+1)t s^\rho/2)\).
Then \(p_t(\cdot\mid x_1)\) is a location-shifted isotropic \(\alpha\)-stable
law with location \(\alpha_t x_1\), scale \(\beta_t\).  In particular, its
radial tail satisfies
\(
    \PP \left(\|X_t-\alpha_t x_1\|_2>r\mid X_1=x_1\right)
    =
    \Theta(r^{-\alpha}),
    \  r\to\infty.
\) 

\noindent\textbf{(iii)} Let $\nu>0$. If \(V_\nu\sim\mathrm{InvGamma}(\nu/2,\nu/2)\) and \(T_t=tV_\nu\), then
\(p_t(\cdot\mid x_1)\) is a multivariate Student-\(t\) law with location
\(\alpha_t x_1\), scale \(\beta_t^2I_d\), \(\nu\) degrees of freedom. In particular, its radial tail satisfies
\(      \PP \left(\|X_t-\alpha_t x_1\|_2>r\mid X_1=x_1\right)
    =
    \Theta(r^{-\nu}),
    \ r\to\infty .
\)
\end{theorem}

\vskip-15pt
\paragraph{Beyond clock-independent straight-line schedules.}
The schedules \(\alpha_t^T,\beta_t^T\) may themselves depend on the realized clock path through its truncated signature; concrete signature-linear and VE-type parameterizations and a near-linear boundary-corrected variant are collected in Appendices~\ref{appendix:signature-linear-schedules}, \ref{appendix:ve-signature-schedule}, and~\ref{appendix:other-clock-paths}. Such clock-dependent schedules generally produce broader Gaussian-mixture sources rather than exact \(\alpha\)-stable or Student-\(t\) laws, so the named recoveries of Theorem~\ref{thm:clock-choice-recovery} are specific to the clock-independent case. A detailed comparison of the random-clock formulation with DLPM, LIM, and Student-\(t\)-based heavy-tailed diffusion families, together with a discussion of data-adaptive clock laws, is given in Appendix~\ref{appendix:hth-comparison}.

\section{Experiments}
\label{sec:experiment}
To assess the effectiveness and generality of the proposed heavy-tailed flow matching model via random clock, we conduct extensive experiments across three benchmarks of increasing dimension and
distinct noise structure.
Full
experimental details are provided in Appendix~\ref{appendix:experimental-setup}.
Ablation studies on flow type, number of function evaluations (NFE), and ODE solver are  deferred to Appendix~\ref{appendix:ablation-studies}.

\subsection{Experimental Setup}
\label{subsec:experiment-setup}
\paragraph{Datasets}
We use three benchmarks of increasing dimension and distinct heavy-tail structure.
\emph{(i)~2D imbalanced $\alpha$-stable mixture~\citep{shariatianheavy}}: a $9$-component isotropic $\alpha$-stable mixture on a $3{\times}3$ grid with imbalanced weights, swept over per-component stability index $\alpha_{D}\in\{1.5,1.8\}$ (smaller $\alpha_{D}$ gives heavier tails).
\emph{(ii)~CIFAR10-LT~\citep{cao2019learning,krizhevsky2009learning}}: a long-tailed subsample of CIFAR-10 at $32{\times}32$ with imbalance ratio $\rho{=}0.01$. CIFAR10-LT also serves as the testbed for baseline comparison and ablations.
\emph{(iii)~HRRR VIL~\citep{benjamin2016north,dowell2022high,pandey2024heavy}}: hourly Vertically Integrated Liquid (VIL) analyses from the High-Resolution Rapid Refresh archive, cropped to $128{\times}128$ over the Central US; VIL is a strongly right-tailed quantity tied to convective weather extremes, following~\citep{pandey2024heavy}.
Sample sizes, splits, augmentations, backbones, optimizers, and samplers are detailed in Appendix~\ref{appendix:experimental-setup}.

\paragraph{Tasks and Metrics}
On  2D $\alpha$-stable
mixtures, we evaluate  mode coverage ability across models using the precision--recall
$f_1^{\mathrm{pr}}$ score~\citep{shariatianheavy}, computed between real
and generated samples.  On CIFAR10-LT, we report Fr\'echet Inception
Distance~\citep{heusel2017gans} (FID), computed against the full
long-tailed training split. On HRRR VIL, we follow
\citep{pathak2022fourcastnet,pandey2024heavy} and evaluate unconditional
generation of VIL fields. For quantitative tail analysis, we compare
flattened generated and test-set intensity samples using three
one-dimensional tail-aware statistics: the kurtosis ratio (KR), skewness
ratio (SR), and a tail-restricted two-sample Kolmogorov--Smirnov statistic
(KS). All results are averaged over 3 random seeds. Formal definitions and implementation details are provided in
Appendix~\ref{appendix:setup-evaluation}.

\paragraph{Methods} Following Theorem~\ref{thm:clock-choice-recovery}, different specifications of clock lead  to different underlying source distribution for HTFM. In the following, we refer to the deterministic-clock variant $T_t{=}t$ as HTFM-G, essentially recovering the standard Gaussian flow matching; the $\alpha/2$-stable subordinator clock as HTFM-$\alpha$;
and the  inverse-Gamma clock $T_t{=}tV_\nu$ for student-$t$, where
$V_\nu\sim\mathrm{InvGamma}(\nu/2,\nu/2)$, as HTFM-$t$. Each variant is
indexed by the polynomial tail-decay rate of its source: $p_\alpha$ for
$\alpha$-stable, $p_\nu$ for Student-$t$, and $p{=}\infty$ for 
Gaussian (since Gaussians essentially have exponential tail decay). Thus, the polynomial tail-decay rate serves as the tail-calibration
parameter across clock families.

\subsection{Cross-Domain Tail Calibration with Random Clocks}
\label{subsec:main-results}

\begin{wraptable}[14]{r}{0.6\textwidth}
\vspace{-1.0em}
\setlength{\abovecaptionskip}{3pt}
\setlength{\belowcaptionskip}{0pt}
\centering
\scriptsize
\setlength{\tabcolsep}{3pt}
\renewcommand{\arraystretch}{1.03}
\resizebox{\linewidth}{!}{%
\begin{tabular}{llcccccc}
\toprule
\multicolumn{2}{c}{Method} 
& \multicolumn{2}{c}{2D synthetic ($f_1^{\mathrm{pr}}\uparrow$)} 
& CIFAR10-LT 
& \multicolumn{3}{c}{HRRR VIL} \\
\cmidrule(lr){1-2}
\cmidrule(lr){3-4}
\cmidrule(lr){5-5}
\cmidrule(lr){6-8}
Variant & $p$ 
& $\alpha_{D}{=}1.5$ 
& $\alpha_{D}{=}1.8$ 
& FID $\downarrow$ 
& KR $\downarrow$ 
& SR $\downarrow$ 
& KS $\downarrow$ \\
\midrule
HTFM-G & $\infty$ 
& 0.888 & 0.893 & 19.29 & 0.72 & 0.68 & 0.42 \\
\midrule
\multirow{3}{*}{HTFM-$\alpha$}
& 1.6 & 0.943 & 0.934 & 16.19 & \textbf{0.04} & 0.11 & 0.11 \\
& 1.7 & 0.941 & 0.923 & 16.10 & 0.47 & \textbf{0.06} & 0.04 \\
& 1.8 & 0.931 & 0.921 & \textbf{15.89} & 0.06 & 0.31 & \textbf{0.03} \\
\midrule
\multirow{3}{*}{HTFM-$t$}
& 1.7 & \textbf{0.962} & \textbf{0.959} & 16.31 & 0.56 & 0.14 & 0.14 \\
& 2.0 & 0.954 & 0.950 & 16.76 & 0.43 & 0.52 & 0.21 \\
& 3.0 & 0.943 & 0.939 & 16.14 & 0.16 & 0.41 & 0.16 \\
\bottomrule
\end{tabular}%
}
\caption{HTFM main results across benchmarks. \(p\) is the polynomial
tail-decay exponent (\(p_\alpha\) for
HTFM-\(\alpha\) and \(p_\nu\) for HTFM-\(t\)). Bold indicates the
best result in each column.}
\vspace{-5em}
\label{tab:main-results-combined}

\end{wraptable}
We first study whether HTFM can act as a
single heavy-tailed flow-matching framework across different datasets. For each benchmark, we keep the training protocol
fixed and vary only the random-clock law and its tail parameter, thereby
comparing the Gaussian clock HTFM-G against the two heavy-tailed clock
families HTFM-$\alpha$ and HTFM-$t$. We report one
sampling budget per benchmark: $f_1^{\mathrm{pr}}$ at NFE${=}10$ on the
2D mixture, FID at NFE${=}100$ on CIFAR10-LT, and HRRR VIL tail
statistics at NFE${=}50$. See Appendix~\ref{appendix:experimental-setup} for detailed information.

\Cref{tab:main-results-combined} reveals two patterns.
\textbf{(i) HTFM with a heavy-tailed clock consistently  outperforms
HTFM-G across three datasets,} which indicates that the random clock provides a
useful inductive bias for heavy-tailed generation.
\textbf{(ii) The tail-decay exponent alone does not determine the best
source family.} At the matched decay rate $p{=}1.7$, HTFM-$t$ performs
best on the 2D mixture, whereas HTFM-$\alpha$ performs better on
CIFAR10-LT and HRRR VIL. Thus, the polynomial tail-decay rate captures
only part of the modeling choice: the \emph{shape} of the source beyond
its asymptotic tail rate, such as the jump structure of $\alpha$-stable
clocks versus the scale-mixture geometry of Student-$t$ clocks, also
interacts with the data distribution. This suggests that the clock
family is a meaningful design knob even after the nominal tail heaviness
is fixed.

\begin{table}[t]
\setlength{\abovecaptionskip}{7pt}
\setlength{\belowcaptionskip}{0pt}
\centering
\small
\setlength{\tabcolsep}{3.6pt}
\renewcommand{\arraystretch}{1.08}
\setlength{\dashlinedash}{1.5pt}
\setlength{\dashlinegap}{1.2pt}
\resizebox{\linewidth}{!}{%
\begin{tabular}{lccc:c:ccccccc:c}
\toprule
\multicolumn{1}{c}{\multirow{2}{*}{Method}}
  & \multicolumn{3}{c}{Noise Support}
  & \multicolumn{1}{c}{\multirow{2}{*}{NFE}}
  & \multicolumn{8}{c}{Polynomial Tail Decay Exponent} \\
\cmidrule(lr){2-4}\cmidrule(lr){6-13}
& $\alpha$-stable & Student-$t$ & Gaussian
& & $p_\alpha=1.6$ & $p_\alpha=1.7$ & $p_\alpha=1.8$ & $p_\alpha=1.9$
& $p_\nu=1.7$ & $p_\nu=2.0$ & $p_\nu=3.0$ & $p=\infty$ \\
\midrule
\multicolumn{4}{l:}{\textit{Baselines}}
& \multicolumn{1}{c:}{}
& \multicolumn{7}{c:}{\textit{Heavy Tailed}}
& \multicolumn{1}{c}{\textit{Gaussian}} \\
\multirow{3}{*}{LIM}
& & &
& 20 & 135.9 & 130.2 & 126.8 & 132.5 & -- & -- & -- & -- \\
& \cmark & \xmark & \xmark
& 100 & 34.70 & 34.11 & 32.40 & 34.10 & -- & -- & -- & -- \\
& & &
& 500 & 20.22 & 19.18 & 18.23 & 20.60 & -- & -- & -- & -- \\
\addlinespace[3.0pt]
\cdashline{1-13}
\addlinespace[3.0pt]
\multirow{3}{*}{DLPM}
& & &
& 20 & 29.16 & 28.55 & 32.49 & 30.96 & -- & -- & -- & 35.49 \\
& \cmark & \xmark & \cmark
& 100 & 18.18 & 18.23 & 18.29 & 18.74 & -- & -- & -- & 22.28 \\
& & &
& 500 & 15.26 & 14.91 & 15.35 & 15.43 & -- & -- & -- & 19.47 \\
\addlinespace[3.0pt]
\cdashline{1-13}
\addlinespace[3.0pt]
\multirow{3}{*}{DLIM}
& & &
& 20 & 21.30 & 23.90 & 22.75 & 24.50 & -- & -- & -- & 25.45 \\
& \cmark & \xmark & \cmark
& 100 & 19.69 & 20.38 & 18.46 & 19.60 & -- & -- & -- & 21.63 \\
& & &
& 500 & 17.22 & 17.73 & 17.90 & 18.14 & -- & -- & -- & 19.92 \\
\addlinespace[3.0pt]
\cdashline{1-13}
\addlinespace[3.0pt]
\multirow{3}{*}{t-Flow}
& & &
& 20 & -- & -- & -- & -- & -- & -- & 22.49 & -- \\
& \xmark & \cmark & \xmark
& 100 & -- & -- & -- & -- & -- & -- & 19.05 & -- \\
& & &
& 500 & -- & -- & -- & -- & -- & -- & 17.66 & -- \\
\midrule
\multicolumn{4}{l:}{\textit{Our Framework}}
& \multicolumn{1}{c:}{}
& \multicolumn{7}{c:}{\textit{Heavy Tailed}}
& \multicolumn{1}{c}{\textit{Gaussian}} \\
\multirow{3}{*}{HTFM}
& & &
& 20 & 17.52 & 17.13 & 16.91 & 16.76 & 19.91 & 20.48 & 19.35 & 22.67 \\
& \cmark & \cmark & \cmark
& 100 & 16.19 & 16.10 & 15.89 & 16.21 & 16.31 & 16.76 & 16.14 & 19.29 \\
& & &
& 500 & 15.27 & 15.04 & 14.99 & 15.67 & 15.97 & 16.51 & 15.38 & 19.04 \\
\bottomrule
\end{tabular}%
}
\caption{Main CIFAR10-LT results. We report FID at each NFE while sweeping the
polynomial tail-decay parameter \(p\). A dash ``--'' marks method--tail pairs
that are not applicable.}
\label{tab:cifar10lt-main-results}
\vspace{-1.5em}
\end{table}

\subsection{Comparison with Baselines}
\label{subsec:cifar10lt-baselines}
In \Cref{tab:cifar10lt-main-results}, we compare HTFM against four heavy-tailed
generative baselines on CIFAR10-LT under matched backbone and training budget: LIM \cite{yoon2023score,popovimproved}, DLPM/DLIM \cite{shariatianheavy}, and t-Flow \cite{pandey2024heavy}. Columns sweep the
polynomial tail-decay rate across the union of the $\alpha$-stable range
$p_\alpha\in\{1.6,1.7,1.8,1.9\}$, the Student-$t$ range
$p_\nu\in\{1.7,2.0,3.0\}$, and the Gaussian limit $p=\infty$.

Three observations can be made.
\textbf{(i) HTFM provides a versatile interface for tail calibration.}
LIM and DLPM/DLIM are tied to $\alpha$-stable laws, while t-Flow is tied
to the finite-variance Student-$t$ regime. In contrast, HTFM recovers the
Gaussian, $\alpha$-stable, and Student-$t$ regimes by changing only the
random-clock law, while keeping the same backbone and training objective.
\textbf{(ii) HTFM inherits the low-NFE efficiency of flow matching.}
At low NFE, HTFM substantially improves over the $\alpha$-stable
diffusion baselines LIM and DLPM; at larger NFE, HTFM and DLPM become
more comparable.
\textbf{(iii) HTFM outperforms the ODE-based heavy-tailed samplers.}
Compared with DLIM and t-Flow, HTFM achieves consistently lower FID
across the reported NFE budgets.

\subsection{Ablation: Signature Order for Clock Conditioning}
\label{subsec:source-covariance-ablation}

This ablation tests whether HTFM's gains come solely from using a heavy-tailed
source, or whether the velocity field also benefits from observing the realized
clock. On CIFAR10-LT, we report FID at NFE\(=100\) while varying the order
\(m\) of the logsignature feature \(\ell_m(T)\) supplied to the network
(cf.~\eqref{eq:chtfm-loss-st}). The case
\(m=0\) denotes no clock feature: the model is still trained with samples from
the random-clock source, but only observes \((X_t,t)\). For \(m\ge 1\), the
model additionally receives logsignature features of the realized clock path,
allowing the vector field to adapt to the corresponding conditional Gaussian
space. For HTFM-\(t\), the affine clock \(T_t=tV_\nu\) is already determined by
first-order path information, so higher-order features are omitted.

\begin{wraptable}[8]{r}{0.42\linewidth}
  \vspace{-14pt}
  \setlength{\abovecaptionskip}{5pt}
  \setlength{\belowcaptionskip}{5pt}
  \centering
  \scriptsize
  \setlength{\tabcolsep}{2.6pt}
  \renewcommand{\arraystretch}{1.05}
  \resizebox{\linewidth}{!}{%
  \begin{tabular}{lcccc}
  \toprule
  Variant & Tail exp. & $m=0$ & $m=1$ & $m=2$ \\
  \midrule
  \multirow{3}{*}{HTFM-$\alpha$}
  & $p_\alpha=1.6$ & 16.78 & \textbf{16.19} & 16.21 \\
  & $p_\alpha=1.7$ & 16.65 & \textbf{16.10} & 16.46 \\
  & $p_\alpha=1.8$ & 16.49 & \textbf{15.89} & 15.94 \\
  \midrule
  \multirow{3}{*}{HTFM-$t$}
  & $p_\nu=1.7$ & 18.72 & \textbf{16.31} & -- \\
  & $p_\nu=2.0$ & 17.30 & \textbf{16.75} & -- \\
  & $p_\nu=3.0$ & 16.58 & \textbf{16.14} & -- \\
  \bottomrule
  \end{tabular}%
  }
  \caption{Signature-order ablation on CIFAR10-LT. Bold indicates the
best result in each row.}
  \label{tab:source-signature-ablation}
\end{wraptable}

Table~\ref{tab:source-signature-ablation} demonstrates the effectiveness of using signature features for clock-conditioning in HTFM. Across all reported HTFM-$\alpha$ and
HTFM-$t$ settings, using $m{=}1$ improves over the clock-independent
$m{=}0$ variant. Meanwhile, increasing the order from $m{=}1$ to $m{=}2$ does not provide a
consistent additional gain, suggesting that first-order clock information
already captures the dominant conditional-scale variation for this
benchmark.

\section{Conclusion}
\label{sec:conclusion}
We introduced Heavy-Tailed Flow Matching via Random Clocks (HTFM), which
represents different families of heavy-tailed sources as mixtures of
clock-conditioned Gaussian sources within a single flow-matching objective.
Using truncated logsignature features, HTFM conditions the velocity field on
the realized clock with negligible overhead. 
Across 2D synthetic data, CIFAR10-LT, and HRRR weather fields, HTFM improves
over the Gaussian-clock variant, enables controllable tail calibration, and
achieves favorable low-NFE performance against heavy-tailed baselines.

\newpage
\bibliography{bib}

@article{pooladian2023multisample,
  title={Multisample flow matching: Straightening flows with minibatch couplings},
  author={Pooladian, Aram-Alexandre and Ben-Hamu, Heli and Domingo-Enrich, Carles and Amos, Brandon and Lipman, Yaron and Chen, Ricky TQ},
  journal={arXiv preprint arXiv:2304.14772},
  year={2023}
}

@article{liu2022rectified,
  title={Rectified flow: A marginal preserving approach to optimal transport},
  author={Liu, Qiang},
  journal={arXiv preprint arXiv:2209.14577},
  year={2022}
}

@article{albergo2022building,
  title={Building normalizing flows with stochastic interpolants},
  author={Albergo, Michael S and Vanden-Eijnden, Eric},
  journal={arXiv preprint arXiv:2209.15571},
  year={2022}
}

@incollection{chevyrev2025primer,
  title={A primer on the signature method in machine learning},
  author={Chevyrev, Ilya and Kormilitzin, Andrey},
  booktitle={Signature Methods in Finance: An Introduction with Computational Applications},
  pages={3--64},
  year={2025},
  publisher={Springer}
}

@article{kidger2019deep,
  title={Deep signature transforms},
  author={Kidger, Patrick and Bonnier, Patric and Perez Arribas, Imanol and Salvi, Cristopher and Lyons, Terry},
  journal={Advances in neural information processing systems},
  volume={32},
  year={2019}
}

@inproceedings{kidger2021signatory,
  title={Signatory: differentiable computations of the signature and logsignature transforms, on both {CPU} and {GPU}},
  author={Kidger, Patrick and Lyons, Terry J.},
  booktitle={International Conference on Learning Representations},
  year={2021},
  url={https://openreview.net/forum?id=lqU2cs3Zca}
}

@inproceedings{morrill2021neural,
  title={Neural Rough Differential Equations for Long Time Series},
  author={Morrill, James and Salvi, Cristopher and Kidger, Patrick and Foster, James},
  booktitle={Proceedings of the 38th International Conference on Machine Learning},
  pages={7829--7838},
  year={2021},
  volume={139},
  series={Proceedings of Machine Learning Research},
  publisher={PMLR},
  url={https://proceedings.mlr.press/v139/morrill21b.html}
}

@article{kiraly2019kernels,
  title={Kernels for Sequentially Ordered Data},
  author={Kiraly, Franz J. and Oberhauser, Harald},
  journal={Journal of Machine Learning Research},
  volume={20},
  number={31},
  pages={1--45},
  year={2019}
}

@article{lyons2014rough,
  title={Rough paths, signatures and the modelling of functions on streams},
  author={Lyons, Terry},
  journal={arXiv preprint arXiv:1405.4537},
  year={2014}
}

@article{levin2013learning,
  title={Learning from the past, predicting the statistics for the future, learning an evolving system},
  author={Levin, Daniel and Lyons, Terry and Ni, Hao},
  journal={arXiv preprint arXiv:1309.0260},
  year={2013}
}

@article{andrews1974scale,
  title={Scale mixtures of normal distributions},
  author={Andrews, D. F. and Mallows, C. L.},
  journal={Journal of the Royal Statistical Society: Series B (Methodological)},
  volume={36},
  number={1},
  pages={99--102},
  year={1974},
  doi={10.1111/j.2517-6161.1974.tb00989.x}
}

@article{barndorff1982normal,
  title={Normal variance-mean mixtures and z-distributions},
  author={Barndorff-Nielsen, O. and Kent, J. and S{\o}rensen, M.},
  journal={International Statistical Review},
  volume={50},
  number={2},
  pages={145--159},
  year={1982},
  doi={10.2307/1402598}
}

@article{kornilov2024optimal,
  title={Optimal flow matching: Learning straight trajectories in just one step},
  author={Kornilov, Nikita and Mokrov, Petr and Gasnikov, Alexander and Korotin, Alexander},
  journal={Advances in Neural Information Processing Systems},
  volume={37},
  pages={104180--104204},
  year={2024}
}

@article{tong2023improving,
  title={Improving and generalizing flow-based generative models with minibatch optimal transport},
  author={Tong, Alexander and Fatras, Kilian and Malkin, Nikolay and Huguet, Guillaume and Zhang, Yanlei and Rector-Brooks, Jarrid and Wolf, Guy and Bengio, Yoshua},
  journal={arXiv preprint arXiv:2302.00482},
  year={2023}
}

@article{song2020score,
  title={Score-based generative modeling through stochastic differential equations},
  author={Song, Yang and Sohl-Dickstein, Jascha and Kingma, Diederik P and Kumar, Abhishek and Ermon, Stefano and Poole, Ben},
  journal={arXiv preprint arXiv:2011.13456},
  year={2020}
}

@article{yoon2023score,
  title={Score-based generative models with L{\'e}vy processes},
  author={Yoon, Eun Bi and Park, Keehun and Kim, Sungwoong and Lim, Sungbin},
  journal={Advances in Neural Information Processing Systems},
  volume={36},
  pages={40694--40707},
  year={2023}
}

@article{pandey2024heavy,
  title={Heavy-tailed diffusion models},
  author={Pandey, Kushagra and Pathak, Jaideep and Xu, Yilun and Mandt, Stephan and Pritchard, Michael and Vahdat, Arash and Mardani, Morteza},
  journal={arXiv preprint arXiv:2410.14171},
  year={2024}
}

@inproceedings{lian2025cauchy,
  title={Cauchy Diffusion: A Heavy-tailed Denoising Diffusion Probabilistic Model for Speech Synthesis},
  author={Lian, Qi and Qi, Yu and Wang, Yueming},
  booktitle={Proceedings of the AAAI Conference on Artificial Intelligence},
  volume={39},
  number={23},
  pages={24549--24557},
  year={2025},
  doi={10.1609/aaai.v39i23.34634}
}

@article{baule2025generative,
  title={Generative modelling with jump-diffusions},
  author={Baule, Adrian},
  journal={arXiv preprint arXiv:2503.06558},
  year={2025}
}

@inproceedings{popovimproved,
  title={Improved Sampling Algorithms for L{\'e}vy-It{\^o} Diffusion Models},
  author={Popov, Vadim and Yermekova, Assel and Sadekova, Tasnima and Khrapov, Artem and Kudinov, Mikhail Sergeevich},
  booktitle={The Thirteenth International Conference on Learning Representations}
}

@inproceedings{shariatianheavy,
  title={Heavy-Tailed Diffusion with Denoising Levy Probabilistic Models},
  author={Shariatian, Dario and Simsekli, Umut and Durmus, Alain Oliviero},
  booktitle={The Thirteenth International Conference on Learning Representations}
}

@article{liu2022flow,
  title={Flow straight and fast: Learning to generate and transfer data with rectified flow},
  author={Liu, Xingchao and Gong, Chengyue and Liu, Qiang},
  journal={arXiv preprint arXiv:2209.03003},
  year={2022}
}

@article{lipman2022flow,
  title={Flow matching for generative modeling},
  author={Lipman, Yaron and Chen, Ricky TQ and Ben-Hamu, Heli and Nickel, Maximilian and Le, Matt},
  journal={arXiv preprint arXiv:2210.02747},
  year={2022}
}

@article{ho2020denoising,
  title={Denoising diffusion probabilistic models},
  author={Ho, Jonathan and Jain, Ajay and Abbeel, Pieter},
  journal={Advances in neural information processing systems},
  volume={33},
  pages={6840--6851},
  year={2020}
}

@article{kovachki2023neural,
  title={Neural operator: Learning maps between function spaces with applications to pdes},
  author={Kovachki, Nikola and Li, Zongyi and Liu, Burigede and Azizzadenesheli, Kamyar and Bhattacharya, Kaushik and Stuart, Andrew and Anandkumar, Anima},
  journal={Journal of Machine Learning Research},
  volume={24},
  number={89},
  pages={1--97},
  year={2023}
}

@article{cuchiero2025universal,
  title={Universal approximation theorems for continuous functions of c{\`a}dl{\`a}g paths and L{\'e}vy-type signature models},
  author={Cuchiero, Christa and Primavera, Francesca and Svaluto-Ferro, Sara},
  journal={Finance and Stochastics},
  volume={29},
  number={2},
  pages={289--342},
  year={2025},
  publisher={Springer}
}

@inproceedings{cao2019learning,
  title={Learning imbalanced datasets with label-distribution-aware margin loss},
  author={Cao, Kaidi and Wei, Colin and Gaidon, Adrien and Arechiga, Nikos and Ma, Tengyu},
  booktitle={Advances in Neural Information Processing Systems},
  year={2019}
}

@inproceedings{liu2019large,
  title={Large-scale long-tailed recognition in an open world},
  author={Liu, Ziwei and Miao, Zhongqi and Zhan, Xiaohang and Wang, Jiayun and Gong, Boqing and Yu, Stella X},
  booktitle={Proceedings of the IEEE/CVF Conference on Computer Vision and Pattern Recognition (CVPR)},
  pages={2537--2546},
  year={2019}
}

@techreport{krizhevsky2009learning,
  title={Learning multiple layers of features from tiny images},
  author={Krizhevsky, Alex and Hinton, Geoffrey},
  year={2009},
  institution={University of Toronto}
}

@article{dowell2022high,
  title={The High-Resolution Rapid Refresh (HRRR): An hourly updating convection-allowing forecast model. Part I: Motivation and system description},
  author={Dowell, David C and Alexander, Curtis R and James, Eric P and Weygandt, Stephen S and Benjamin, Stanley G and Manikin, Geoffrey S and Blake, Benjamin T and Brown, John M and Olson, Joseph B and Hu, Ming and others},
  journal={Weather and Forecasting},
  volume={37},
  number={8},
  pages={1371--1395},
  year={2022}
}

@article{pathak2022fourcastnet,
  title={{FourCastNet}: A global data-driven high-resolution weather model using adaptive {F}ourier neural operators},
  author={Pathak, Jaideep and Subramanian, Shashank and Harrington, Peter and Raja, Sanjeev and Chattopadhyay, Ashesh and Mardani, Morteza and Kurth, Thorsten and Hall, David and Li, Zongyi and Azizzadenesheli, Kamyar and Hassanzadeh, Pedram and Kashinath, Karthik and Anandkumar, Animashree},
  journal={arXiv preprint arXiv:2202.11214},
  year={2022}
}

@article{grundemann2022rarest,
  title={Rarest rainfall events will see the greatest relative increase in magnitude under future climate change},
  author={Gr{\"u}ndemann, Gaby J and van de Giesen, Nick and Brunner, Lukas and van der Ent, Ruud},
  journal={Communications Earth \& Environment},
  volume={3},
  number={1},
  pages={235},
  year={2022},
  doi={10.1038/s43247-022-00558-8}
}

@article{cont2001empirical,
  title={Empirical properties of asset returns: stylized facts and statistical issues},
  author={Cont, Rama},
  journal={Quantitative Finance},
  volume={1},
  number={2},
  pages={223--236},
  year={2001},
  doi={10.1080/713665670}
}

@book{coles2001introduction,
  title={An introduction to statistical modeling of extreme values},
  author={Coles, Stuart and Bawa, Joanna and Trenner, Lesley and Dorazio, Pat},
  volume={208},
  year={2001},
  publisher={Springer}
}

@book{resnick2007heavy,
  title={Heavy-tail phenomena: probabilistic and statistical modeling},
  author={Resnick, Sidney I},
  year={2007},
  publisher={Springer}
}

@book{embrechts2013modelling,
  title={Modelling extremal events: for insurance and finance},
  author={Embrechts, Paul and Kl{\"u}ppelberg, Claudia and Mikosch, Thomas},
  volume={33},
  year={2013},
  publisher={Springer Science \& Business Media}
}

@article{seneviratne2021weather,
  title={Weather and climate extreme events in a changing climate},
  author={Seneviratne, Sonia I and Zhang, Xuebin and Adnan, Muhammad and Badi, Wafae and Dereczynski, Claudine and Luca, A Di and Ghosh, Subimal and Iskandar, Iskhaq and Kossin, James and Lewis, Sophie and others},
  journal={Climate change 2021: The physical science basis: Working group I contribution to the sixth assessment report of the intergovernmental panel on climate change},
  pages={1513--1766},
  year={2021},
  publisher={Cambridge University Press}
}

@article{cont2026tail,
  title={Tail-gan: Learning to simulate tail risk scenarios},
  author={Cont, Rama and Cucuringu, Mihai and Xu, Renyuan and Zhang, Chao},
  journal={Management Science},
  volume={72},
  number={4},
  pages={2917--2936},
  year={2026},
  publisher={INFORMS}
}

@inproceedings{zhang2024long,
  title={Long-tailed diffusion models with oriented calibration},
  author={Zhang, Tianjiao and Zheng, Huangjie and Yao, Jiangchao and Wang, Xiangfeng and Zhou, Mingyuan and Zhang, Ya and Wang, Yanfeng},
  booktitle={The twelfth international conference on learning representations},
  year={2024}
}

@inproceedings{kingma2014auto,
  title={Auto-Encoding Variational Bayes},
  author={Kingma, Diederik P. and Welling, Max},
  booktitle={International Conference on Learning Representations},
  year={2014},
  url={https://openreview.net/forum?id=33X9fd2-9FyZd}
}

@inproceedings{goodfellow2014generative,
  title={Generative Adversarial Nets},
  author={Goodfellow, Ian and Pouget-Abadie, Jean and Mirza, Mehdi and Xu, Bing and Warde-Farley, David and Ozair, Sherjil and Courville, Aaron and Bengio, Yoshua},
  booktitle={Advances in Neural Information Processing Systems},
  volume={27},
  year={2014}
}

@inproceedings{jaini2020tails,
  title={Tails of {L}ipschitz Triangular Flows},
  author={Jaini, Priyank and Kobyzev, Ivan and Yu, Yaoliang and Brubaker, Marcus},
  booktitle={Proceedings of the 37th International Conference on Machine Learning},
  pages={4673--4681},
  year={2020},
  volume={119},
  series={Proceedings of Machine Learning Research},
  publisher={PMLR},
  url={https://proceedings.mlr.press/v119/jaini20a.html}
}

@inproceedings{hickling2025flexible,
  title={Flexible Tails for Normalizing Flows},
  author={Hickling, Tennessee and Prangle, Dennis},
  booktitle={Proceedings of the 42nd International Conference on Machine Learning},
  pages={23155--23178},
  year={2025},
  volume={267},
  series={Proceedings of Machine Learning Research},
  publisher={PMLR},
  url={https://proceedings.mlr.press/v267/hickling25a.html}
}

@article{benjamin2016north,                                                                                                                                                                                               
    title={A North American hourly assimilation and model forecast cycle: The Rapid Refresh},                                                                                                                               
    author={Benjamin, Stanley G and Weygandt, Stephen S and Brown, John M and ...},                                                                                                                                         
    journal={Monthly Weather Review},                                                                                                                                                                                     
    volume={144}, number={4}, pages={1669--1694}, year={2016}                                                                                                                                                               
}

@article{massey1951kolmogorov,
  title={The {K}olmogorov--{S}mirnov test for goodness of fit},
  author={Massey Jr, Frank J},
  journal={Journal of the American Statistical Association},
  volume={46},
  number={253},
  pages={68--78},
  year={1951}
}

@inproceedings{heusel2017gans,
  title={GANs Trained by a Two Time-Scale Update Rule Converge to a Local Nash Equilibrium},
  author={Heusel, Martin and Ramsauer, Hubert and Unterthiner, Thomas and Nessler, Bernhard and Hochreiter, Sepp},
  booktitle={Advances in Neural Information Processing Systems},
  year={2017}
}
\bibliographystyle{plain}
\newpage
\appendix
\section*{Appendix Contents}
{\small
\noindent\textbf{A\quad \hyperref[appendix:additional_discussions]{Additional Discussions}}\dotfill \pageref{appendix:additional_discussions}\\
\hspace*{2em}A.1\quad \hyperref[appendix:related-work]{Related work}\dotfill \pageref{appendix:related-work}\\
\hspace*{2em}A.2\quad \hyperref[appendix:hth-comparison]{Detailed comparison with existing heavy-tailed models}\dotfill \pageref{appendix:hth-comparison}\\
\hspace*{2em}A.3\quad \hyperref[appendix:gsm-examples]{Representative Gaussian scale mixtures}\dotfill \pageref{appendix:gsm-examples}\\
\hspace*{2em}A.4\quad \hyperref[appendix:gmvm]{Gaussian mean-variance mixtures and asymmetric extensions}\dotfill \pageref{appendix:gmvm}\\
\hspace*{2em}A.5\quad \hyperref[appendix:other-clock-paths]{Other admissible clock-conditioned paths}\dotfill \pageref{appendix:other-clock-paths}\\[2pt]
\noindent\textbf{B\quad \hyperref[appendix:diffusion-extension]{All-clock extension to diffusion models}}\dotfill \pageref{appendix:diffusion-extension}\\[2pt]
\noindent\textbf{C\quad \hyperref[appendix:additional_results]{Additional results and proofs of main theorems}}\dotfill \pageref{appendix:additional_results}\\
\hspace*{2em}C.1\quad \hyperref[appendix:clock-history-note]{A note on conditioning on the clock history}\dotfill \pageref{appendix:clock-history-note}\\
\hspace*{2em}C.2\quad \hyperref[appendix:proof-clock-affine-law]{Conditional Gaussian law of the affine path}\dotfill \pageref{appendix:proof-clock-affine-law}\\
\hspace*{2em}C.3\quad \hyperref[appendix:proof-chtfm-htfm-equivalence]{Detailed pathwise equivalence result and proof}\dotfill \pageref{appendix:proof-chtfm-htfm-equivalence}\\
\hspace*{2em}C.4\quad \hyperref[appendix:proof-clock-straight-geodesic]{Clock-conditioned straight-line geodesic}\dotfill \pageref{appendix:proof-clock-straight-geodesic}\\
\hspace*{2em}C.5\quad \hyperref[appendix:proof-clock-choice-recovery]{Proof of Theorem~\ref{thm:clock-choice-recovery}}\dotfill \pageref{appendix:proof-clock-choice-recovery}\\[2pt]
\noindent\textbf{D\quad \hyperref[appendix:log-sig]{Additional Information on Log-Signature}}\dotfill \pageref{appendix:log-sig}\\
\hspace*{2em}D.1\quad \hyperref[appendix:signature-linear-schedules]{Signature-linear schedules}\dotfill \pageref{appendix:signature-linear-schedules}\\
\hspace*{2em}D.2\quad \hyperref[appendix:ve-signature-schedule]{VE-type signature-linear schedule}\dotfill \pageref{appendix:ve-signature-schedule}\\[2pt]
\noindent\textbf{E\quad \hyperref[appendix:experimental-setup]{Experimental Setup}}\dotfill \pageref{appendix:experimental-setup}\\
\hspace*{2em}E.1\quad \hyperref[appendix:setup-dataset]{Dataset}\dotfill \pageref{appendix:setup-dataset}\\
\hspace*{2em}E.2\quad \hyperref[appendix:setup-baselines]{Baselines}\dotfill \pageref{appendix:setup-baselines}\\
\hspace*{2em}E.3\quad \hyperref[appendix:setup-evaluation]{Evaluation}\dotfill \pageref{appendix:setup-evaluation}\\
\hspace*{2em}E.4\quad \hyperref[appendix:setup-denoiser]{Denoiser Architecture}\dotfill \pageref{appendix:setup-denoiser}\\
\hspace*{2em}E.5\quad \hyperref[appendix:setup-training]{Training}\dotfill \pageref{appendix:setup-training}\\
\hspace*{2em}E.6\quad \hyperref[appendix:setup-sampling]{Sampling}\dotfill \pageref{appendix:setup-sampling}\\[2pt]
\noindent\textbf{F\quad \hyperref[appendix:ablation-studies]{Ablation Studies}}\dotfill \pageref{appendix:ablation-studies}\\
\hspace*{2em}F.1\quad \hyperref[appendix:flow-type-ablation]{Effect of Flow Type}\dotfill \pageref{appendix:flow-type-ablation}\\
\hspace*{2em}F.2\quad \hyperref[appendix:nfe-ablation]{Effect of Number of Function Evaluations (NFE)}\dotfill \pageref{appendix:nfe-ablation}\\
\hspace*{2em}F.3\quad \hyperref[appendix:solver-ablation]{Effect of ODE Solver}\dotfill \pageref{appendix:solver-ablation}\\[2pt]
\noindent\textbf{G\quad \hyperref[appendix:visualizations]{Visualizations}}\dotfill \pageref{appendix:visualizations}\par
}
\bigskip

\section{Additional Discussions}
\label{appendix:additional_discussions}

\subsection{Related work}
\label{appendix:related-work}

\paragraph{Heavy-tailed diffusion and flow models.}
Most diffusion and flow-matching generative models use Gaussian reference
distributions or Gaussian perturbation kernels. Several recent works replace
this Gaussian mechanism by heavy-tailed alternatives. L\'evy score models use
\(\alpha\)-stable L\'evy processes and replace the ordinary score by a
fractional score, leading to nonlocal reverse-time dynamics and deterministic
fractional probability-flow samplers
\citep{yoon2023score,popovimproved}. DLPM and DLIM instead build a
discrete-time \(\alpha\)-stable analogue of DDPM/DDIM; their key practical step
is a Gaussian scale-mixture augmentation that makes the conditional reverse
kernels Gaussian while the marginal noising law remains heavy-tailed
\citep{shariatianheavy}. Student-\(t\)-based diffusion and flow models replace
Gaussian perturbation kernels with Student-\(t\) kernels and derive
family-specific denoising posteriors and training divergences, including
t-Flow as a heavy-tailed flow construction \citep{pandey2024heavy}. Cauchy
Diffusion studies the Cauchy case in speech synthesis, using Cauchy noise and a
ratio-of-Gaussians view to support stochastic and deterministic sampling in a
DDPM-style vocoder \citep{lian2025cauchy}. Related jump-diffusion generative
models consider finite-activity jumps and marginal-matching reverse dynamics
\citep{baule2025generative}. Our random-clock formulation follows the same
broad motivation of using heavy-tailed perturbation families, but differs in
where the tractability is imposed: after conditioning on a realized clock path,
the flow-matching problem is Gaussian and uses the ordinary conditional velocity
target, while marginalizing the clock produces the heavy-tailed source or path
family.

\paragraph{Latent Gaussian mixtures and random time changes.}
Gaussian scale mixtures provide a classical route from conditionally Gaussian
models to heavy-tailed marginals \citep{andrews1974scale}. More general normal
variance-mean mixtures can also be interpreted as Brownian motions with drift
stopped at independent random times, and include many skewed and heavy-tailed
families \citep{barndorff1982normal}. This latent-Gaussian perspective is
exactly the mechanism exploited by DLPM for \(\alpha\)-stable noising
\citep{shariatianheavy}. Our construction extends this idea from a scalar
mixing variable to a path-valued clock. The clock can determine the conditional
source covariance, the interpolation schedule, or both. Thus, the model keeps a
Gaussian conditional regression problem but allows the outer clock law to encode
family choice, tail heaviness, and path-dependent allocation of noise.

\paragraph{Signature and logsignature transforms.}
Path signatures originate in rough path theory and give an order-sensitive
sequence of iterated-integral features for streams
\citep{lyons2014rough,chevyrev2025primer}. In statistics and machine learning,
signature features have been used as universal, truncated representations for
stream-valued inputs \citep{levin2013learning}, as kernels for sequential data
\citep{kiraly2019kernels}, and as differentiable neural-network components
\citep{kidger2019deep,kidger2021signatory}. Logsignature features are
especially relevant when one wants a compact nonredundant representation of path
increments; for example, neural rough differential equations summarize long
driving signals over subintervals by logsignatures
\citep{morrill2021neural}. Recent approximation results also justify
signature-linear representations of path functionals on suitable compact path
sets, including c\`adl\`ag settings \citep{cuchiero2025universal}. Our use of
signatures is narrower: the clock path is not the generated sample path, but a
latent conditioning environment. We use low-order logsignature features to give
the neural velocity field finite-dimensional access to the realized random
clock, and, optionally, to parameterize clock-dependent affine schedules.

\subsection{Detailed comparison with existing heavy-tailed models}
\label{appendix:hth-comparison}

The related-work overview in Appendix~\ref{appendix:related-work} surveys
heavy-tailed diffusion and flow constructions. Here we elaborate on the
technical contrast between the random-clock formulation and the closest prior
families.

\paragraph{DLPM and L\'evy-based diffusion.}
DLPM \citep{shariatianheavy} obtains heavy-tailed diffusion transitions via
\(\alpha\)-stable innovations and a robust outer power applied to empirical
losses. As shown in Theorem~\ref{thm:clock-choice-recovery}(ii), our
stable-subordinator clock with the area-clock covariance reproduces the same
\(\alpha\)-stable marginal source family. The two constructions differ at two
levels. \emph{At the flow-matching target level}, DLPM averages over noise
realizations and learns a clock-marginal vector field, whereas conditioning on
the realized clock path turns the regression into a clock-indexed family of
pathwise optima \(u_t^T\); a network with access to the clock can learn the
appropriate optimum for each noise realization, and the random-clock view can
therefore be read as a refinement of the DLPM stable-noise construction with
the same marginal family but a better-specified flow-matching target. We
expect empirical gains when different clock realizations induce meaningfully
different optima, and the formulation also avoids direct use of the generally
unavailable \(\alpha\)-stable density. \emph{At the loss level}, the outer
robust power \(z\mapsto z^r\) with \(0<r<1\) used by DLPM destroys the
additive projection decomposition in
Proposition~\ref{prop:chtfm-htfm-equivalence}; in contrast, the random-clock
formulation keeps the inner problem in a conditionally Gaussian
\(L^2\)-projection space and treats any robustification over sampled clocks
as an estimator-level stabilization rather than a change of the population
target. Compared with L\'evy score and L\'evy--It\^o diffusion models
\citep{yoon2023score,popovimproved}, the flow-matching construction further
avoids a nonlocal fractional-score target: after conditioning on the clock,
the endpoint-conditioned target remains the explicit affine expression
\eqref{eq:clock-affine-target}.

\paragraph{Student-\(t\) diffusion and \(t\)-Flow.}
Student-\(t\) diffusion and \(t\)-Flow models replace Gaussian perturbations
by Student-\(t\) heavy-tailed perturbations \citep{pandey2024heavy}. Their
construction is tailored to the Student-\(t\) family: it derives
Student-\(t\)-specific perturbation kernels and uses a \(\gamma\)-divergence
objective, with the practical formulation restricted to the finite-variance
regime \(\nu>2\). Extending that derivation formally to \(\nu\le2\) is
mathematically problematic, rather than merely an implementation detail,
because the covariance-based normalizations and finite-variance quantities
used by the Student-\(t\) perturbation model cease to be well-defined. In
contrast, the inverse-gamma clock in
Theorem~\ref{thm:clock-choice-recovery}(iii) is well-defined for every
\(\nu>0\), so our pathwise construction covers the extremely heavy-tailed
infinite-variance regime \(\nu\le2\) without changing the conditional Gaussian
target.

\paragraph{Toward data-adaptive clocks.}
The most general use of random clocks would not prescribe the clock law in
advance. Instead, one could estimate a clock distribution from data, or infer
a data-dependent posterior over clock paths with an amortized model, and then
use the inferred clock as the conditioning variable in
\eqref{eq:chtfm-loss-st}. This would turn the clock into a learned latent
description of local tail behavior rather than a hand-chosen family such as a
stable subordinator or inverse-gamma slope. Developing such data-adaptive
clock inference requires additional identifiability and estimation
assumptions, so we leave it as a future extension.

\subsection{Representative Gaussian scale mixtures}
\label{appendix:gsm-examples}

We recall two standard Gaussian scale-mixture examples that are only mentioned
briefly in Section~\ref{subsec:ht-noise}. For \(\alpha\in(0,2]\), a centered
symmetric isotropic \(\alpha\)-stable random vector \(S_\alpha\in\RR^d\) is
characterized by
\begin{equation}
\EE\left[e^{i\langle \xi,S_\alpha\rangle}\right]
=
\exp\bigl(-\|\xi\|_2^\alpha/2\bigr),
\qquad \xi\in\RR^d.
\label{eq:stable-char}
\end{equation}
The case \(\alpha=2\) reduces to the Gaussian law, whereas for
\(\alpha<2\) the law is heavy-tailed, with radial tail probability
\[
\PP\left(\|S_\alpha\|_2>r\right)\asymp r^{-\alpha},
\qquad r\to\infty.
\]
With this normalization, \(S_\alpha\) admits the Gaussian scale-mixture
representation
\[
S_\alpha
\overset{d}{=}
\sqrt{V_\alpha}\,G,
\]
where \(G\sim\cN(0,I_d)\) and \(V_\alpha\) is a positive
\(\alpha/2\)-stable mixing variable independent of \(G\), with Laplace
transform
\[
\EE[e^{-sV_\alpha}]
=
\exp\left(-2^{\alpha/2-1}s^{\alpha/2}\right),
\qquad s\ge 0.
\]

Likewise, a centered \(d\)-dimensional Student-\(t\) random vector \(Z_\nu\)
with \(\nu>0\) degrees of freedom and identity scale has density
\begin{equation}
f_{\nu,d}(x)
=
\frac{\Gamma\bigl((\nu+d)/2\bigr)}
{\Gamma(\nu/2)\,(\nu\pi)^{d/2}}
\left(1+\frac{\|x\|_2^2}{\nu}\right)^{-(\nu+d)/2},
\qquad x\in\RR^d,
\label{eq:student-t-density}
\end{equation}
and satisfies
\[
\PP\left(\|Z_\nu\|_2>r\right)\asymp r^{-\nu},
\qquad r\to\infty.
\]
It can be written as
\[
Z_\nu
\overset{d}{=}
\sqrt{V_\nu}\,G,
\qquad
V_\nu \sim \mathrm{InvGamma}(\nu/2,\nu/2),
\]
with \(V_\nu\) independent of \(G\). These examples make explicit why
conditioning on the latent scale leaves a Gaussian problem, while marginalizing
the latent scale produces polynomial tails.

\subsection{Gaussian mean-variance mixtures and asymmetric extensions}
\label{appendix:gmvm}

The main text focuses on centered Gaussian scale mixtures, which already cover a
broad class of symmetric heavy-tailed laws. A natural extension is to allow the
latent variable to affect both the conditional mean and the conditional
covariance, leading to the class of \emph{Gaussian mean-variance mixtures}.

Let \(V\) be a latent random variable taking values in a measurable space
\(\mathsf V\), and let \(G\sim\cN(0,I_d)\) be independent of \(V\). A random
vector \(X\in\RR^d\) is called a Gaussian mean-variance mixture if it admits a
representation of the form
\begin{equation}
X
=
m(V)+L(V)G,
\label{eq:gmvm-representation}
\end{equation}
where \(m:\mathsf V\to\RR^d\) and \(L:\mathsf V\to\RR^{d\times d}\) are
measurable. Equivalently,
\begin{equation}
X\mid V
\sim
\cN\bigl(m(V),\,\Sigma(V)\bigr),
\qquad
\Sigma(V):=L(V)L(V)^\top.
\label{eq:gmvm-conditional}
\end{equation}
The heavy-tailed or asymmetric nature of the marginal law is produced by
marginalizing over the latent variable \(V\), while conditional on \(V\) the
distribution remains Gaussian.

The centered Gaussian scale-mixture subclass from the main text is recovered by
setting \(m(V)\equiv 0\). When \(m(V)\neq 0\), the latent variable affects not
only the covariance but also the location, and the resulting marginal law may
be asymmetric. Thus Gaussian mean-variance mixtures provide a natural route to
skewed heavy-tailed extensions.

In the random-clock setting, the corresponding extension would be to replace the
centered conditional source law
\[
X_0^T
\sim
\cN\bigl(0,\Sigma(T_{[0,1]})\bigr)
\]
by the more general form
\begin{equation}
X_0^T
\sim
\cN\bigl(m(T_{[0,1]}),\,\Sigma(T_{[0,1]})\bigr),
\label{eq:clock-gmvm}
\end{equation}
where \(m:\mathbb D([0,1];\RR_+)\to\RR^d\) is a measurable mean functional.
Under the affine interpolation
\[
X_t^T=\alpha_tX_1^T+\beta_tX_0^T,
\]
we then have
\begin{equation}
X_t^T\mid X_1^T=x_1
\sim
\cN\bigl(\alpha_t x_1+\beta_t m(T_{[0,1]}),\,
\beta_t^2\Sigma(T_{[0,1]})\bigr).
\label{eq:clock-gmvm-xt}
\end{equation}
Hence the path remains conditionally Gaussian after conditioning on the clock
path, so the latent-Gaussian principle underlying the flow-matching
construction continues to hold.

Moreover, the explicit affine target formula is unaffected by this extension.
Indeed, for a.e.\ \(t\in[0,1)\), the identity
\[
X_0^T=\frac{X_t^T-\alpha_tX_1^T}{\beta_t}
\]
still holds pathwise, so substituting into
\[
\frac{\dd X_t^T}{\dd t}
=
\frac{\dd\alpha_t}{\dd t}X_1^T
+
\frac{\dd\beta_t}{\dd t}X_0^T
\]
yields
\begin{equation}
\EE\left[
\frac{\dd X_t^T}{\dd t}
\,\middle|\,
X_t^T=x,\ X_1^T=x_1
\right]
=
\frac{\frac{\dd\beta_t}{\dd t}}{\beta_t}\,x
+
\left(
\frac{\dd\alpha_t}{\dd t}
-
\frac{\frac{\dd\beta_t}{\dd t}}{\beta_t}\alpha_t
\right)x_1.
\label{eq:gmvm-target}
\end{equation}
Thus, even in the more general mean-variance mixture setting, the affine
clock-conditioned target retains the same closed-form expression as in the
centered case. What changes is the induced marginal family after integrating
out the clock, which may now be skewed as well as heavy-tailed.

This extension suggests that the random-clock framework is not limited to
symmetric Gaussian scale mixtures. By allowing a nonzero conditional mean
functional \(m(T_{[0,1]})\), one can in principle accommodate a broader class
of asymmetric heavy-tailed laws while preserving conditional Gaussianity. We do
not pursue this direction further in the main text.

\subsection{Other admissible clock-conditioned paths}
\label{appendix:other-clock-paths}

The main text focuses on the straight-line interpolation, which is the most
common choice in conditional flow matching and already suffices for the heavy-tailed framework developed here. More generally, however, the same
clock-conditioned construction applies to a broader class of affine paths.

Let \(\alpha,\beta:[0,1]\to\RR\) be absolutely continuous functions such that
\[
\alpha_0=0,\qquad
\beta_0=1,\qquad
\alpha_1=1,\qquad
\beta_1=0,
\]
and assume \(\beta_t>0\) for all \(t\in[0,1)\). Define
\begin{equation}
X_t^T=\alpha_t X_1^T+\beta_t X_0^T,
\qquad t\in[0,1].
\label{eq:appendix-clock-affine-path}
\end{equation}
As in Section~\ref{subsec:flow-choices}, if
\(
X_0^T\sim \cN(0,\Sigma(T_{[0,1]}))
\)
and \(X_0^T\) and \(X_1^T\) are conditionally independent given \(T_{[0,1]}\),
then
\begin{equation}
p_t^T(\cdot\mid x_1)
=
\Law(X_t^T\mid X_1^T=x_1)
=
\cN\bigl(\alpha_t x_1,\ \beta_t^2\Sigma(T_{[0,1]})\bigr).
\label{eq:appendix-clock-affine-law}
\end{equation}
Moreover, for a.e.\ \(t\in[0,1)\),
\[
\frac{\dd X_t^T}{\dd t}
=
\frac{\dd\alpha_t}{\dd t} X_1^T
+
\frac{\dd\beta_t}{\dd t} X_0^T,
\qquad
X_0^T=\frac{X_t^T-\alpha_t X_1^T}{\beta_t},
\]
so the corresponding conditional flow-matching target is
\begin{equation}
v_t^T(x\mid x_1)
=
\frac{\frac{\dd\beta_t}{\dd t}}{\beta_t}\,x
+
\left(
\frac{\dd\alpha_t}{\dd t}
-
\frac{\frac{\dd\beta_t}{\dd t}}{\beta_t}\alpha_t
\right)x_1.
\label{eq:appendix-clock-affine-target}
\end{equation}
Thus any admissible pair \((\alpha,\beta)\) yields an explicit
clock-conditioned path/target pair.

\paragraph{VP-type path.}
A natural diffusion-inspired choice is the variance-preserving schedule. Let
\(\bar\alpha:[0,1]\to[0,1]\) be absolutely continuous with
\[
\bar\alpha_0=0,\qquad \bar\alpha_1=1,
\]
and assume \(\bar\alpha_t<1\) for \(t\in[0,1)\). Define
\[
\alpha_t=\bar\alpha_t,
\qquad
\beta_t=\sqrt{1-\bar\alpha_t^2}.
\]
Then
\begin{equation}
X_t^T=\bar\alpha_t X_1^T+\sqrt{1-\bar\alpha_t^2}\,X_0^T,
\label{eq:appendix-vp-path}
\end{equation}
which is the analogue of a variance-preserving interpolation. In this case,
\eqref{eq:appendix-clock-affine-target} becomes
\begin{equation}
v_t^T(x\mid x_1)
=
-\frac{\bar\alpha_t\frac{\dd\bar\alpha_t}{\dd t}}{1-\bar\alpha_t^2}\,x
+
\frac{\frac{\dd\bar\alpha_t}{\dd t}}{1-\bar\alpha_t^2}\,x_1.
\label{eq:appendix-vp-target}
\end{equation}

\paragraph{VE-type path.}
A variance-exploding analogue can also be accommodated within the affine
framework. Let \(\sigma:[0,1]\to[0,1]\) be an absolutely continuous schedule
such that
\[
\sigma_0=1,\qquad \sigma_1=0,\qquad \sigma_t>0\qquad \text{for } t\in[0,1).
\]
Set
\[
\alpha_t=1-\sigma_t,
\qquad
\beta_t=\sigma_t.
\]
Then the corresponding affine interpolation is
\begin{equation}
X_t^T=(1-\sigma_t)X_1^T+\sigma_t X_0^T,
\qquad t\in[0,1].
\label{eq:appendix-ve-path}
\end{equation}
Conditioning on \(X_1^T=x_1\) and the realized clock path \(T_{[0,1]}\), we have
\begin{equation}
p_t^T(\cdot\mid x_1)
=
\Law(X_t^T\mid X_1^T=x_1)
=
\cN\bigl((1-\sigma_t)x_1,\ \sigma_t^2\Sigma(T_{[0,1]})\bigr).
\label{eq:appendix-ve-law}
\end{equation}
Thus the conditional variance is directly controlled by the scale schedule
\(\sigma_t^2\). Although our interpolation is written in the generative
direction from source to data, this corresponds to a variance-exploding
schedule when viewed backward from \(t=1\) to \(t=0\).

Differentiating \eqref{eq:appendix-ve-path} gives
\[
\frac{\dd X_t^T}{\dd t}
=
-\frac{\dd\sigma_t}{\dd t} X_1^T
+
\frac{\dd\sigma_t}{\dd t} X_0^T,
\qquad
X_0^T=\frac{X_t^T-(1-\sigma_t)X_1^T}{\sigma_t}.
\]
Substituting into \eqref{eq:appendix-clock-affine-target}, or directly into
\(\frac{\dd X_t^T}{\dd t}\), yields the conditional flow-matching target
\begin{equation}
v_t^T(x\mid x_1)
=
\frac{\frac{\dd\sigma_t}{\dd t}}{\sigma_t}(x-x_1),
\qquad t\in[0,1).
\label{eq:appendix-ve-target}
\end{equation}

A simple concrete choice is the polynomial schedule
\[
\sigma_t=(1-t)^\gamma,
\qquad \gamma>0,
\]
for which
\[
X_t^T=\bigl(1-(1-t)^\gamma\bigr)X_1^T+(1-t)^\gamma X_0^T,
\]
and
\begin{equation}
v_t^T(x\mid x_1)
=
-\frac{\gamma}{1-t}(x-x_1).
\label{eq:appendix-ve-target-poly}
\end{equation}
More generally, any absolutely continuous decreasing schedule \(\sigma_t\) with
the above endpoint conditions yields an admissible VE-type path within the same
clock-conditioned affine framework.

\section{All-clock extension to diffusion models}
\label{appendix:diffusion-extension}

In this appendix we record how the random-clock viewpoint extends to
score-based diffusion models. The point of this subsection is slightly different
from the flow-matching construction in the main text. Flow matching only needs
a clock-conditioned interpolation and an explicit velocity target; it does not
require reversing a stochastic process. Diffusion models do require a
reverse-time dynamics, and for general c\`adl\`ag clocks this dynamics is not, in
general, a purely local score SDE. The correct all-clock statement is a
reverse-kernel formulation. The familiar local score SDE is recovered on the
continuous part of the clock, and in particular for absolutely continuous
clocks.

Let \(W\) be a standard \(d\)-dimensional Brownian motion independent of the
clock \(T\). Given drift \(f:[0,1]\times\RR^d\to\RR^d\) and deterministic
diffusion matrix \(G:[0,1]\to\RR^{d\times d}\), consider the clock-conditioned
forward dynamics
\begin{equation}
\dd X_t
=
f_t(X_t)\,\dd t
+
G_t\,\dd W_{T_t},
\qquad
X_0\sim p_0(\cdot\mid T_{[0,1]}).
\label{eq:clock-forward-sde}
\end{equation}
For a fixed realized clock path, \(W_{T_t}\) is a semimartingale with quadratic
variation \(T_t\). If \(T\) is continuous, this is a continuous martingale; if
\(T\) has jumps, then \(W_{T_t}\) has Gaussian jumps with covariance proportional
to the clock jump sizes. Write \(a_t:=G_tG_t^\top\) and
\(p_t^T:=\Law(X_t\mid T_{[0,1]})\). Whenever \(p_t^T\) has a density, we define
the clock-conditioned score by
\begin{equation}
s_t^T(x):=\nabla_x\log p_t^T(x).
\label{eq:clock-score}
\end{equation}
As in the main text, the density is conditioned on the full clock path. This is
the right object for clock-dependent initial laws such as
\(X_0=L(T_{[0,1]})G\); in general it cannot be replaced by conditioning only on
\(T_{[0,t]}\).

We first state the reverse dynamics in a form that is valid for arbitrary
realized c\`adl\`ag clock paths. For \(0\le s<t\le1\), let
\(P_{s,t}^T(x,\dd y)\) denote the conditional forward transition kernel of
\eqref{eq:clock-forward-sde} from time \(s\) to time \(t\), given the full clock
path. Thus \(p_s^T P_{s,t}^T=p_t^T\). The reverse transition from \(t\) back to
\(s\) is the Bayes kernel associated with this joint law.

\begin{theorem}[All-clock reverse dynamics in kernel form]
\label{thm:clock-reverse-pf}
\label{thm:clock-reverse-kernel}
Fix a realized clock path \(T\in\mathbb D([0,1];\RR_+)\). Assume that the forward
equation \eqref{eq:clock-forward-sde} is well posed, that regular conditional
transition kernels \(P_{s,t}^T\) exist, and that \(p_s^T P_{s,t}^T=p_t^T\) for
\(0\le s<t\le1\). For \(s<t\), define a reverse kernel
\(R_{t,s}^T(y,\dd x)\) by the measure identity
\begin{equation}
p_t^T(\dd y)\,R_{t,s}^T(y,\dd x)
=
p_s^T(\dd x)\,P_{s,t}^T(x,\dd y)
\label{eq:reverse-bayes-kernel}
\end{equation}
on \(\RR^d\times\RR^d\). Equivalently, when transition densities exist,
\begin{equation}
R_{t,s}^T(y,\dd x)
=
\frac{p_s^T(x)p_{s,t}^T(y\mid x)}{p_t^T(y)}\,\dd x.
\label{eq:reverse-bayes-density}
\end{equation}
Let \((Y_r)_{r\in[0,1]}\) be the reverse-time Markov process initialized by
\(Y_0\sim p_1^T\) and, for \(0\le r<r'\le1\), using the transition
\begin{equation}
\Law(Y_{r'}\in \dd x\mid Y_r=y,T_{[0,1]})
=
R_{1-r,\,1-r'}^T(y,\dd x).
\label{eq:reverse-kernel-process}
\end{equation}
Then the reverse process is conditional marginal-matching:
\begin{equation}
\Law(Y_r\mid T_{[0,1]})=p_{1-r}^T,
\qquad r\in[0,1].
\label{eq:kernel-reverse-marginal-match}
\end{equation}
\end{theorem}

\begin{proof}
Fix \(0\le s<t\le1\). Integrating \eqref{eq:reverse-bayes-kernel} over
\(y\in\RR^d\) gives, for every measurable set \(A\subseteq\RR^d\),
\[
\int_{\RR^d} R_{t,s}^T(y,A)\,p_t^T(\dd y)
=
\int_A P_{s,t}^T(x,\RR^d)\,p_s^T(\dd x)
=
p_s^T(A).
\]
Thus a single reverse step transports \(p_t^T\) back to \(p_s^T\). Taking
\(t=1-r\) and \(s=1-r'\), this says that if \(Y_r\sim p_{1-r}^T\), then the
transition \eqref{eq:reverse-kernel-process} gives
\(Y_{r'}\sim p_{1-r'}^T\). Since \(Y_0\sim p_1^T\), iterating this identity over
any finite partition of \([0,1]\) yields \(\Law(Y_r\mid T_{[0,1]})=p_{1-r}^T\)
for every \(r\in[0,1]\). The Markov kernels are consistent because they are the
regular conditional kernels of the same forward Markov law read backward in
time, so the corresponding reverse process exists by the usual extension
construction.
\end{proof}

The theorem is deliberately stated without a score. A score is a local object,
whereas a jump of the clock produces a finite, nonlocal update of the law. This
can be seen explicitly. Let \(t\) be a jump time of the realized clock and set
\(h:=\Delta T_t=T_t-T_{t-}>0\). Since the drift term has no atom in physical time,
the forward jump is
\[
X_t=X_{t-}+G_t Z_t,
\qquad Z_t\sim\cN(0,hI_d),
\]
so the jump kernel is \(Q_t^T(x,\dd y)=\cN(y;x,h a_t)(\dd y)\). Therefore
\[
p_t^T(\dd y)=\int Q_t^T(x,\dd y)\,p_{t-}^T(\dd x),
\]
and the reverse jump is the Bayes kernel characterized by
\begin{equation}
p_t^T(\dd y)R_{t,t-}^T(y,\dd x)
=
p_{t-}^T(\dd x)Q_t^T(x,\dd y).
\label{eq:reverse-jump-kernel-measure}
\end{equation}
When \(a_t\) is nonsingular, this becomes the density formula
\begin{equation}
R_{t,t-}^T(y,\dd x)
=
\frac{p_{t-}^T(x)\varphi_{h a_t}(y-x)}
{\int p_{t-}^T(z)\varphi_{h a_t}(y-z)\,\dd z}\,\dd x,
\label{eq:reverse-jump-kernel-density}
\end{equation}
where \(\varphi_{h a_t}\) is the Gaussian density with covariance \(h a_t\). Thus
reverse jumps are posterior sampling steps. They cannot be represented exactly
by the local drift term \(a_t s_t^T\,\dd T_t\).

We next recover the usual score-based formula from the kernel theorem on the
continuous part of the clock. Write
\[
T_t^c:=T_t-\sum_{0<s\le t}\Delta T_s
\]
for the continuous part of the realized clock, and define the reversed
continuous clock \(\widehat T_r^c:=T_1^c-T_{1-r}^c\). At reverse times
corresponding to jumps of \(T\), the process uses the jump kernel
\eqref{eq:reverse-jump-kernel-measure}; between these jumps, or if the clock is
continuous, the reverse dynamics have the following local representation.

\begin{corollary}[Local score form on the continuous clock part]
\label{cor:clock-local-score}
Assume the hypotheses of Theorem~\ref{thm:clock-reverse-kernel}. In addition,
assume that the continuous part of the forward law admits strictly positive
smooth densities, that the usual decay and integrability conditions for
integration by parts hold, and that the local reverse equations below are well
posed. On the continuous clock part, with \(\tau=1-r\), the reverse process has
the local score representation
\begin{equation}
\dd Y_r
=
-f_\tau(Y_r)\,\dd r
+
a_\tau s_\tau^T(Y_r)\,\dd\widehat T_r^c
+
G_\tau\,\dd\bar W_{\widehat T_r^c},
\label{eq:clock-reverse-sde-thm}
\end{equation}
where \(\bar W\) is a standard Brownian motion. If \(T\) is continuous, this
local equation alone gives \(\Law(Y_r\mid T_{[0,1]})=p_{1-r}^T\) for all
\(r\in[0,1]\). If \(T\) has jumps, it must be supplemented by the reverse jump
kernels \eqref{eq:reverse-jump-kernel-measure} at the corresponding reverse
jump times.

For the continuous part, the deterministic probability-flow equation is
\begin{equation}
\dd Z_r
=
-f_\tau(Z_r)\,\dd r
+
\frac12 a_\tau s_\tau^T(Z_r)\,\dd\widehat T_r^c.
\label{eq:clock-pf-ode-thm}
\end{equation}
When \(T\) has jumps, an exact deterministic flow across a jump additionally
requires choosing a transport map from \(p_t^T\) to \(p_{t-}^T\); this choice is
not canonical and is not determined by the score alone.
\end{corollary}

\begin{proof}
We prove the continuous part; the jump updates are already identified by
Theorem~\ref{thm:clock-reverse-kernel}. Let \(\psi\in C_c^\infty(\RR^d)\). Along
the continuous part of the forward dynamics, the conditional law satisfies the
weak equation
\[
\dd\int \psi(x)p_t^T(x)\,\dd x
=
\int f_t(x)\cdot\nabla\psi(x)p_t^T(x)\,\dd x\,\dd t
+
\frac12\int \Tr(a_t\nabla^2\psi(x))p_t^T(x)\,\dd x\,\dd T_t^c.
\]
Set \(q_r=p_{1-r}^T\) and \(\tau=1-r\). Since
\(\widehat T_r^c=T_1^c-T_{1-r}^c\), this forward weak equation gives
\[
\dd\int \psi(x)q_r(x)\,\dd x
=
-\int f_\tau(x)\cdot\nabla\psi(x)q_r(x)\,\dd x\,\dd r
-
\frac12\int \Tr(a_\tau\nabla^2\psi(x))q_r(x)\,\dd x\,\dd\widehat T_r^c.
\]
Now let \(\rho_r=\Law(Y_r\mid T_{[0,1]})\) for the local SDE
\eqref{eq:clock-reverse-sde-thm}. Its weak equation is
\begin{align*}
\dd\int \psi(x)\rho_r(x)\,\dd x
&=
-\int f_\tau(x)\cdot\nabla\psi(x)\rho_r(x)\,\dd x\,\dd r \\
&\quad+
\int a_\tau s_\tau^T(x)\cdot\nabla\psi(x)\rho_r(x)\,\dd x\,\dd\widehat T_r^c \\
&\quad+
\frac12\int \Tr(a_\tau\nabla^2\psi(x))\rho_r(x)\,\dd x\,\dd\widehat T_r^c.
\end{align*}
Substituting \(\rho_r=q_r=p_\tau^T\) and using
\(p_\tau^T s_\tau^T=\nabla p_\tau^T\), integration by parts gives
\[
\int a_\tau s_\tau^T\cdot\nabla\psi\,p_\tau^T\,\dd x
=
\int a_\tau\nabla p_\tau^T\cdot\nabla\psi\,\dd x
=
-\int \Tr(a_\tau\nabla^2\psi)p_\tau^T\,\dd x.
\]
Therefore the two \(\dd\widehat T_r^c\)-terms combine to
\(-\frac12\int \Tr(a_\tau\nabla^2\psi)p_\tau^T\,\dd x\,\dd\widehat T_r^c\),
which matches the weak equation for \(q_r\). With matching initial law at
\(r=0\), uniqueness gives the stated marginal matching.

For the deterministic equation \eqref{eq:clock-pf-ode-thm}, the weak equation
has the same drift term but only
\(\frac12\int a_\tau s_\tau^T\cdot\nabla\psi\,\zeta_r\,\dd x\,\dd\widehat T_r^c\)
in the clock direction. Substituting \(\zeta_r=q_r\) and applying the same
integration by parts yields exactly the same weak equation for \(q_r\). Hence
\(\Law(Z_r\mid T_{[0,1]})=p_{1-r}^T\) on the continuous part.
\end{proof}

If the clock is absolutely continuous, \(T_t=\int_0^t\ell_s\,\dd s\), then
\(\dd\widehat T_r^c=\ell_\tau\,\dd r\), and
\eqref{eq:clock-reverse-sde-thm} reduces to
\[
\dd Y_r
=
\bigl[-f_\tau(Y_r)+\ell_\tau a_\tau s_\tau^T(Y_r)\bigr]\dd r
+
\sqrt{\ell_\tau}\,G_\tau\,\dd\bar B_r,
\qquad \tau=1-r,
\]
while the probability flow becomes
\[
\dd Z_r
=
\bigl[-f_\tau(Z_r)+\tfrac12\ell_\tau a_\tau s_\tau^T(Z_r)\bigr]\dd r.
\]
This is the classical reverse-SDE/probability-flow structure with the diffusion
rate modulated by the realized clock rate.

The learning target associated with the local continuous dynamics is the
clock-conditioned score \(s_t^T\). Accordingly, for continuous-clock diffusion
implementations, or for the continuous part of a jump-clock implementation, a
pathwise score-matching objective is
\begin{equation}
\mathcal L_{\mathrm{DSM}}(\theta)
:=
\EE_{t\sim\mathrm{Unif}(0,1)}
\omega_t\,\EE_T\left[
\EE_{X_t\mid T}\left[
\bigl\|
s_\theta(X_t,t,\eta_T)-s_t^T(X_t)
\bigr\|_2^2
\right]
\right],
\label{eq:clock-diffusion-loss}
\end{equation}
where \(\omega_t>0\) is an optional weighting schedule. For clocks with jumps,
this score objective does not by itself specify the reverse jump kernels in
\eqref{eq:reverse-jump-kernel-measure}; exact reverse simulation for such clocks
requires those nonlocal kernels, or an additional approximation to them.

In practice, the marginal score \(s_t^T\) is typically unavailable. One may
therefore use a clock-conditioned denoising score-matching objective whenever
the forward transition kernel given \((X_0,T_{[0,1]})\) is tractable:
\begin{equation}
\begin{aligned}
&\mathcal L_{\mathrm{CDSM}}(\theta)\\
:=&
\EE_{t\sim\mathrm{Unif}(0,1)}
\omega_t\,\EE_{X_0,T}\left[
\EE_{X_t\mid X_0,T}\left[
\bigl\|
s_\theta(X_t,t,\eta_T)
-
\nabla_{x_t}\log p_t(x_t\mid X_0,T_{[0,1]})\big|_{x_t=X_t}
\bigr\|_2^2
\right]
\right].
\end{aligned}
\label{eq:clock-cdsm-loss}
\end{equation}
Under the same square-integrability and projection conditions as standard
denoising score matching, the conditional score in
\eqref{eq:clock-cdsm-loss} projects onto the marginal score \(s_t^T\) after
conditioning on \((X_t,T_{[0,1]})\).

For linear-Gaussian forward schedules, the conditional kernel is Gaussian. If
\[
X_t\mid (X_0,T_{[0,1]})
\sim
\cN\bigl(m_t(X_0,T_{[0,1]}),\,Q_t(T_{[0,1]})\bigr),
\]
then
\[
\nabla_{x_t}\log p_t(x_t\mid X_0,T_{[0,1]})
=
-\,Q_t(T_{[0,1]})^{-1}
\bigl(x_t-m_t(X_0,T_{[0,1]})\bigr),
\]
and \eqref{eq:clock-cdsm-loss} becomes
\begin{equation}
\begin{aligned}
&\mathcal L_{\mathrm{CDSM}}(\theta)\\
:=&
\EE_{t\sim\mathrm{Unif}(0,1)}
\omega_t\,\EE_{X_0,T}\left[
\EE_{X_t\mid X_0,T}\left[
\Bigl\|
s_\theta(X_t,t,\eta_T)
+
Q_t(T_{[0,1]})^{-1}
\bigl(X_t-m_t(X_0,T_{[0,1]})\bigr)
\Bigr\|_2^2
\right]
\right].
\end{aligned}
\label{eq:clock-cdsm-loss-gaussian}
\end{equation}

\section{Additional results and proofs of main theorems}
\label{appendix:additional_results}

\subsection{A note on conditioning on the clock history}
\label{appendix:clock-history-note}

The main construction conditions on the full realized clock path
\(T_{[0,1]}\). In general this conditioning cannot be replaced by conditioning
only on the past clock history \(T_{[0,t]}\), because the clock-conditioned
source may itself depend on future clock values such as \(T_1\). The following
lemma records the additional non-anticipativity condition under which such a
reduction is valid.

\begin{lemma}[Clock-history reduction under non-anticipative initialization]
\label{lem:clock-history}
Consider the forward construction \eqref{eq:clock-forward-sde}. Assume that,
for every \(t\in[0,1]\), the initial variable \(X_0\) is conditionally
independent of \(\sigma(T_{(t,1]})\) given \(\sigma(T_{[0,t]})\), and that the
Brownian motion driving the SDE is independent of the clock and of the base
randomness used to form \(X_0\). Then, for every \(t\in[0,1]\),
\[
\Law(X_t\mid T_{[0,1]})=\Law(X_t\mid T_{[0,t]})
\qquad\text{a.s.}
\]
Equivalently, whenever both sides admit conditional densities,
\[
p_t(\cdot\mid T_{[0,1]})=p_t(\cdot\mid T_{[0,t]})
\qquad\text{a.s.}
\]
\end{lemma}

\begin{proof}
Fix \(t\in[0,1]\) and set
\[
\mathcal G_t
:=
\sigma(X_0,T_{[0,t]})\vee \sigma(W_u:0\le u\le T_t).
\]
By non-anticipativity of the forward dynamics, \(X_t\) is measurable with
respect to \(\mathcal G_t\). Since \(T_t\) is \(\sigma(T_{[0,t]})\)-measurable
and the Brownian motion \(W\) is independent of the entire clock,
\(\sigma(W_u:0\le u\le T_t)\) is conditionally independent of
\(\sigma(T_{(t,1]})\) given \(\sigma(T_{[0,t]})\). By assumption the same is
true for \(X_0\). Hence \(\mathcal G_t\), and therefore \(X_t\), is
conditionally independent of \(\sigma(T_{(t,1]})\) given \(\sigma(T_{[0,t]})\).
For any bounded measurable test function \(\varphi\),
\[
\EE\left[\varphi(X_t)\,\middle|\,\sigma(T_{[0,1]})\right]
=
\EE\left[\varphi(X_t)\,\middle|\,\sigma(T_{[0,t]})\right]
\qquad\text{a.s.},
\]
which proves the claim.
\end{proof}

The non-anticipativity assumption in Lemma~\ref{lem:clock-history} is essential.
For example, if \(X_0\mid T_{[0,1]}\sim \cN(0,T_1I_d)\), then already at
\(t=0\) the conditional law given \(T_{[0,1]}\) is \(\cN(0,T_1I_d)\), whereas
the conditional law given \(T_{[0,0]}\) is the corresponding mixture over
\(T_1\). These laws are generally different. This is why the main
flow-matching construction consistently uses
\(p_t^T=\Law(X_t^T)=\Law(X_t\mid T_{[0,1]})\) rather than replacing it by a
past-clock conditional law.

\subsection{Conditional Gaussian law of the affine path}
\label{appendix:proof-clock-affine-law}

We verify the conditional Gaussian law \eqref{eq:clock-affine-law} stated in
the main text. Recall the assumptions: conditional on the realized clock path
\(T_{[0,1]}\), \(X_0^T\sim\cN(0,\Sigma(T_{[0,1]}))\); \(X_0^T\) and \(X_1^T\)
are conditionally independent given \(T_{[0,1]}\); and the affine path is
\(X_t^T=\alpha_t^T X_1^T+\beta_t^T X_0^T\) with deterministic schedules
\(\alpha_t^T,\beta_t^T\) given \(T_{[0,1]}\).

\begin{proof}[Proof of \eqref{eq:clock-affine-law}]
Fix \(T_{[0,1]}\) and \(X_1^T=x_1\). Under this conditioning, \(\alpha_t^T\)
and \(\beta_t^T\) are deterministic, and the conditional independence
assumption gives
\(
X_0^T\mid (X_1^T=x_1,\,T_{[0,1]})\sim\cN(0,\Sigma(T_{[0,1]})).
\)
Hence
\[
X_t^T\mid (X_1^T=x_1,\,T_{[0,1]})
=
\alpha_t^T x_1+\beta_t^T\,X_0^T\,\bigl|\,(X_1^T=x_1,\,T_{[0,1]})
\]
is an affine transformation of a Gaussian, with mean \(\alpha_t^T x_1\) and
covariance \((\beta_t^T)^2\Sigma(T_{[0,1]})\). Therefore
\[
X_t^T\,\bigl|\,(X_1^T=x_1,\,T_{[0,1]})
\sim
\cN\bigl(\alpha_t^T x_1,\,(\beta_t^T)^2\Sigma(T_{[0,1]})\bigr),
\]
which is precisely \eqref{eq:clock-affine-law}.
\end{proof}

\subsection{Detailed pathwise equivalence result and proof}
\label{appendix:proof-chtfm-htfm-equivalence}

\begin{proposition}[Detailed pathwise conditional--marginal equivalence]
\label{prop:chtfm-htfm-equivalence-full}
Assume \eqref{eq:clock-compatibility-velocity} and suppose that, for every
\(\theta\) in the parameter class and for a.e.\ \(t\in[0,1]\),
\[
\EE\left[
\|u_\theta(X_t^T,t,T_{[0,1]})\|_2^2
+
\|v_t^T(X_t^T\mid X_1^T)\|_2^2
\,\middle|\,
T_{[0,1]}
\right]
<\infty
\qquad\text{a.s.}
\]
Then, for a.e.\ \(t\), there exists a nonnegative
\(\sigma(T_{[0,1]})\)-measurable random variable \(C_t(T)\), independent of
\(\theta\), such that
\[
\EE\left[
R_t^{\mathrm{cond}}(\theta;X_1^T,T)
\,\middle|\,
T_{[0,1]}
\right]
=
R_t^{\mathrm{marg}}(\theta;T)+C_t(T).
\]
Consequently, whenever the risks and \(C_t(T)\) are integrable over \((t,T)\)
with \(t\sim\mathrm{Unif}(0,1)\),
\[
\mathcal L_{\mathrm{CHTFM}}(\theta)
=
\mathcal L_{\mathrm{HTFM}}(\theta)
+
\EE_{t\sim\mathrm{Unif}(0,1)}
\EE_T\left[C_t(T)\right].
\]
In particular, if the additive constant is finite, the two population
objectives have the same minimizers. If \(u_\theta\) is differentiable in
\(\theta\) and differentiation may be exchanged with the conditional
expectations, the clock expectation, and the uniform-time expectation, then
\[
\nabla_\theta\mathcal L_{\mathrm{CHTFM}}(\theta)
=
\nabla_\theta\mathcal L_{\mathrm{HTFM}}(\theta).
\]
\end{proposition}

\begin{proof}[Proof of Propositions~\ref{prop:chtfm-htfm-equivalence}
and~\ref{prop:chtfm-htfm-equivalence-full}]
Fix a time \(t\) for which the assumptions hold. By
\eqref{eq:clock-compatibility-velocity},
\[
u_t^T(x)
=
\EE\left[
v_t^T(X_t^T\mid X_1^T)
\,\middle|\,
X_t^T=x
\right]
\]
for \(p_t^T\)-a.e.\ \(x\), where expectations are taken in the fixed-clock
conditional probability space. The same identity is understood in the weak
sense when densities are not used. Thus, within the clock-conditioned
environment indexed by \(T\), \(u_t^T(X_t^T)\) is the \(L^2\)-projection of the
pathwise conditional velocity \(v_t^T(X_t^T\mid X_1^T)\) onto the sigma-field
\(\sigma(X_t^T)\). Applying the conditional Pythagorean identity gives
\begin{align*}
\EE\left[
R_t^{\mathrm{cond}}(\theta;X_1^T,T)
\,\middle|\,
T_{[0,1]}
\right]
&=
R_t^{\mathrm{marg}}(\theta;T)
+
\EE\left[
\bigl\|
v_t^T(X_t^T\mid X_1^T)-u_t^T(X_t^T)
\bigr\|_2^2
\,\middle|\,
T_{[0,1]}
\right] \\
&=
R_t^{\mathrm{marg}}(\theta;T)+C_t(T),
\end{align*}
where
\[
C_t(T)
:=
\EE\left[
\bigl\|
v_t^T(X_t^T\mid X_1^T)-u_t^T(X_t^T)
\bigr\|_2^2
\,\middle|\,
T_{[0,1]}
\right]
\ge 0.
\]

If the risks and \(C_t(T)\) are integrable over \((t,T)\) with
\(t\sim\mathrm{Unif}(0,1)\), then Fubini's theorem yields
\begin{equation}
\mathcal L_{\mathrm{CHTFM}}(\theta)
=
\mathcal L_{\mathrm{HTFM}}(\theta)
+
\EE_{t\sim\mathrm{Unif}(0,1)}
\EE_T\left[C_t(T)\right].
\label{eq:objective-decomp}
\end{equation}
The last term is independent of \(\theta\). Hence, whenever it is finite, the
two population objectives have the same minimizers. If \(u_\theta\) is
differentiable in \(\theta\) and differentiation may be exchanged with the
conditional expectations, the clock expectation, and the uniform-time
expectation, then
\(\nabla_\theta\mathcal L_{\mathrm{CHTFM}}(\theta)
=\nabla_\theta\mathcal L_{\mathrm{HTFM}}(\theta)\).
\end{proof}

\subsection{Clock-conditioned straight-line geodesic}\label{appendix:proof-clock-straight-geodesic}

\begin{theorem}[Clock-conditioned straight-line geodesic]
\label{thm:clock-straight-geodesic}
Fix \(p\in[1,\infty)\). For a.e.\ realized clock path \(T\), assume
\(p_0(\cdot\mid T_{[0,1]}),p_{\mathrm{data}}\in\mathcal P_p(\RR^d)\), and let
\(\pi^T\) be a \(p\)-optimal coupling between them. Then
\(
\mu_t^T:=\bigl((1-t)x_0+t x_1\bigr)_{\#}\pi^T,\  t\in[0,1],
\)
is a constant-speed \(W_p\)-geodesic from \(p_0(\cdot\mid T_{[0,1]})\) to
\(p_{\mathrm{data}}\).
\end{theorem}

\begin{proof}
Fix a realized clock path \(T\) for which
\(p_0(\cdot\mid T_{[0,1]})\in\mathcal P_p(\RR^d)\), and abbreviate
\[
\mu_0:=p_0(\cdot\mid T_{[0,1]}),
\qquad
\mu_1:=p_{\mathrm{data}},
\qquad
\mu_r:=\bigl((1-r)x_0+r x_1\bigr)_{\#}\pi^T,
\quad r\in[0,1].
\]
Then \(\mu_t^T=\mu_t\). By assumption \(\mu_0,\mu_1\in\mathcal P_p(\RR^d)\),
so the \(p\)-optimal transport problem is well posed and
\(W_p(\mu_0,\mu_1)<\infty\).

For any \(r,r'\in[0,1]\), the pushforward
\[
\bigl((1-r)x_0+r x_1,\,(1-r')x_0+r'x_1\bigr)_{\#}\pi^T
\]
is a coupling of \(\mu_r\) and \(\mu_{r'}\). Hence, using
the optimality of \(\pi^T\) at the last step,
\begin{equation}
W_p(\mu_r,\mu_{r'})^p
\le
\int \|(r-r')(x_1-x_0)\|^p\,d\pi^T
=
|r-r'|^p\,W_p(\mu_0,\mu_1)^p.
\label{eq:interp-upper}
\end{equation}
Now fix \(0\le r\le r'\le 1\). Applying \eqref{eq:interp-upper} to the pairs
\((0,r)\), \((r,r')\), and \((r',1)\), we obtain
\[
\begin{aligned}
&W_p(\mu_0,\mu_r)+W_p(\mu_r,\mu_{r'})+W_p(\mu_{r'},\mu_1) \\
\le\;&
r\,W_p(\mu_0,\mu_1)
+(r'-r)\,W_p(\mu_0,\mu_1)
+(1-r')\,W_p(\mu_0,\mu_1) \\
=\;&
W_p(\mu_0,\mu_1).
\end{aligned}
\]
On the other hand, the triangle inequality yields
\[
W_p(\mu_0,\mu_1)
\le
W_p(\mu_0,\mu_r)+W_p(\mu_r,\mu_{r'})+W_p(\mu_{r'},\mu_1).
\]
Therefore equality holds throughout, and in particular,
\[
W_p(\mu_r,\mu_{r'})
=
|r-r'|\,W_p(\mu_0,\mu_1)
\qquad\text{for all } r,r'\in[0,1].
\]
Thus \((\mu_t^T)_{t\in[0,1]}\) is a constant-speed \(W_p\)-geodesic between
\(p_0(\cdot\mid T_{[0,1]})\) and \(p_{\mathrm{data}}\).
\end{proof}

\begin{remark}
\label{rem:choice-of-p}
For each fixed clock path with finite covariance matrix,
\(p_0(\cdot\mid T_{[0,1]})=\cN(0,\Sigma(T_{[0,1]}))\) belongs to
\(\mathcal P_q(\RR^d)\) for every \(q\ge 1\). Thus
Theorem~\ref{thm:clock-straight-geodesic} is pathwise in \(T\): after
conditioning on the clock, the source side does not restrict the choice of
\(p\), and the admissible \(p\) is governed by the data moment assumption. This
does not imply that the clock-marginal source law has all moments; after
integrating out \(T\), stable or Student-\(t\) clock choices can again produce
heavy-tailed sources with limited or infinite moments. The classical
\(W_2\)-statement is recovered when the data law has a finite second moment and
the argument is applied pathwise.
\end{remark}

\subsection{Proof of Theorem~\ref{thm:clock-choice-recovery}}
\label{appendix:proof-clock-choice-recovery}

\begin{proof}[Proof of Theorem~\ref{thm:clock-choice-recovery}]
By \eqref{eq:clock-affine-law} and the source covariance \(\Sigma_A(T_{[0,1]}):=2A(T)I_d\), where
\(A(T):=\int_0^1 T_s\,\dd s\), for every
fixed clock path and endpoint \(X_1^T=x_1\),
\[
X_t^T\mid X_1^T=x_1
\sim
\cN\bigl(\alpha_t x_1,\,2\beta_t^2A(T)I_d\bigr),
\qquad
A(T)=\int_0^1T_s\,\dd s .
\]
Thus the clock-marginal conditional law is a Gaussian scale mixture.

For the deterministic clock \(T_t=t\), \(A(T)=1/2\), so the covariance above is
\(\beta_t^2I_d\). This proves the Gaussian case.

For the stable case, let \(\alpha\in(0,2)\), \(\rho=\alpha/2\), and let \(T\) be the
\(\rho\)-stable subordinator in the theorem. With
\(A(T)=\int_0^1T_s\dd s\), the integration-by-parts identity for
Lebesgue--Stieltjes integrals gives
\[
A(T)=\int_0^1(1-s)\,\dd T_s.
\]
Using the independent increments of the subordinator, for any \(\lambda\ge0\),
\begin{align}
\EE\left[e^{-\lambda A(T)}\right]
&=
\exp\left(
-\frac{\rho+1}{2}\int_0^1(\lambda(1-s))^\rho\,\dd s
\right)
=
\exp\left(-\frac{\lambda^\rho}{2}\right).
\label{eq:area-stable-laplace}
\end{align}
Therefore, for any \(\xi\in\RR^d\),
\begin{align}
\EE\left[e^{i\langle \xi,X_t\rangle}\mid X_1=x_1\right]
&=
\exp\bigl(i\alpha_t\langle \xi,x_1\rangle\bigr)\,
\EE_T\left[
\exp\left(-\beta_t^2 A(T)\|\xi\|_2^2\right)
\right] \notag\\
&=
\exp\bigl(i\alpha_t\langle \xi,x_1\rangle\bigr)\,
\exp\left(-\frac{\beta_t^\alpha\|\xi\|_2^\alpha}{2}\right).
\label{eq:stable-char-derivation-xt}
\end{align}
This is the characteristic function of a location-shifted isotropic
\(\alpha\)-stable law with location \(\alpha_t x_1\) and scale \(\beta_t\), under
the normalization used in the main text. The radial tail asymptotics
\(\PP(\|X_t-\alpha_t x_1\|_2>r\mid X_1=x_1)=\Theta(r^{-\alpha})\) as \(r\to\infty\)
then follow from the standard tail behavior of isotropic \(\alpha\)-stable laws~\citep{resnick2007heavy}.

For the Student-\(t\) case, \(T_t=tV_\nu\) gives \(A(T)=V_\nu/2\), and hence
\(\Sigma_A(T_{[0,1]})=V_\nu I_d\). Since
\(V_\nu\sim\mathrm{InvGamma}(\nu/2,\nu/2)\), marginalizing over \(V_\nu\) yields
\begin{align}
p_{X_t\mid X_1=x_1}(x)
&=
\int_0^\infty
\frac{1}{(2\pi \beta_t^2 v)^{d/2}}
\exp\left(
-\frac{\|x-\alpha_t x_1\|_2^2}{2\beta_t^2 v}
\right)
\frac{(\nu/2)^{\nu/2}}{\Gamma(\nu/2)}
v^{-\nu/2-1}e^{-\nu/(2v)}\,\dd v \notag\\
&=
\frac{\Gamma((\nu+d)/2)}
{\Gamma(\nu/2)(\nu\pi)^{d/2}\beta_t^d}
\left(
1+\frac{\|x-\alpha_t x_1\|_2^2}{\nu\beta_t^2}
\right)^{-(\nu+d)/2}.
\label{eq:student-derivation-xt}
\end{align}
This is the multivariate Student-\(t\) density with location \(\alpha_t x_1\),
scale \(\beta_t^2I_d\), and \(\nu\) degrees of freedom, whose radial tail
satisfies \(\PP(\|X_t-\alpha_t x_1\|_2>r\mid X_1=x_1)=\Theta(r^{-\nu})\) as
\(r\to\infty\)~\citep{resnick2007heavy}.
\end{proof}


\section{Additional Information on Log-Signature}
\label{appendix:log-sig}

This appendix fixes the signature-feature convention used in
Section~\ref{subsec:clock-signature-schedules}. The main paper specializes the
clock feature to the truncated logsignature, \(\eta_T=\ell_m(T)\). For
completeness we keep the generic notation \(\eta_T\) throughout this appendix
and discuss alternative choices (raw signatures, lead--lag transforms, or
concatenations with simple summaries). The feature \(\eta_T\) has two roles:
it conditions the learned vector field \(u_\theta(X_t^T,t,\eta_T)\), and, when
a clock-dependent trajectory is used, it may also parameterize the
interpolation schedules \(\alpha_t^T\) and \(\beta_t^T\). Throughout this
appendix, the feature is a summary of the full realized clock path
\(T_{[0,1]}\), sampled before drawing the clock-conditioned source and before
solving the generation ODE.

Let \(\mathbf z:[0,1]\to\RR^q\) be a continuous path of bounded variation. Its
signature is the tensor series
\[
S(\mathbf z)
=
\left(
1,
\int_{0<u_1<1}\dd \mathbf z_{u_1},
\int_{0<u_1<u_2<1}\dd \mathbf z_{u_1}\otimes \dd \mathbf z_{u_2},
\ldots
\right),
\]
and its truncation to order \(m\) is denoted by \(S^{(\le m)}(\mathbf z)\). The
truncated logsignature is
\[
\mathrm{LogSig}^{(\le m)}(\mathbf z)
:=
\log_{\otimes}\bigl(S^{(\le m)}(\mathbf z)\bigr),
\]
where \(\log_{\otimes}\) is computed in the truncated tensor algebra. Unlike the
raw signature, the logsignature lies in the truncated free Lie algebra and
removes many algebraic redundancies among tensor coordinates.

In our application, \(q=2\) and the path is the time-augmented clock
\(\mathbf T_s=(s,T_s)\). Since the random clocks considered in the paper may be
c\`adl\`ag and may have jumps, we do not rely on a canonical signature of the
jump path itself. Instead, for a discretization
\(0=t_0<t_1<\cdots<t_N=1\), we form the sampled time-augmented path
\[
\bigl(\mathbf T_{t_0},\mathbf T_{t_1},\ldots,\mathbf T_{t_N}\bigr)
=
\bigl((t_0,T_{t_0}),(t_1,T_{t_1}),\ldots,(t_N,T_{t_N})\bigr)
\]
and compute signatures after linearly interpolating these sampled points. With
this convention,
\[
\phi_m(T):=S^{(\le m)}(\mathbf T^{\mathrm{lin}})
\]
is the truncated signature feature used in
\eqref{eq:clock-truncated-signature-feature}, where \(\mathbf T^{\mathrm{lin}}\)
denotes the piecewise linear interpolation of the sampled time-augmented clock.
A logsignature version of the same feature is
\[
\ell_m(T):=
\mathrm{LogSig}^{(\le m)}(\mathbf T^{\mathrm{lin}}).
\]
Thus the model feature \(\eta_T\) may be chosen as \(\phi_m(T)\), as
\(\ell_m(T)\), or as a concatenation of either representation with simple clock
summaries such as \(T_1\) or \(\int_0^1T_s\dd s\). When invoking the
signature-linear approximation result in
Proposition~\ref{prop:signature-linear-universality}, one should take the raw
signature coordinates \(\phi_m(T)\), because the proposition is stated for linear
functionals of signatures. In neural-network conditioning, we often use the
logsignature \(\ell_m(T)\), since it is more compact; using logsignature
coordinates in the learned schedules should therefore be understood as a
practical finite-dimensional parameterization, not as the literal raw-signature
universality statement.

Under the piecewise linear convention, a jump in the sampled clock is represented
by a steep linear segment between two consecutive observation times. Other
conventions, such as inserting explicit vertical jump segments or using a
lead--lag transform, would define different finite-dimensional features. They do
not change the flow-matching identities in the main text: after conditioning on
\(T_{[0,1]}\), the path remains Gaussian as in
\eqref{eq:clock-affine-law}, and the endpoint-conditioned target remains the
affine expression \eqref{eq:clock-affine-target}. The feature \(\eta_T\) only
summarizes the realized clock environment for the network and, when used, for
the schedule parameterization.

The first logsignature level is the increment of the time-augmented path, so it
records \((1,T_1-T_0)\), which is \((1,T_1)\) under the standing assumption
\(T_0=0\). Higher levels encode how clock mass is allocated over physical time.
For example, in two dimensions the second Lie level corresponds to the signed
area between the physical-time coordinate and the clock coordinate; for monotone
clocks this separates paths with the same terminal value \(T_1\) but different
temporal allocations of their increments. This is why low-order signature
information can represent path-dependent functionals such as the area-clock
quantity \(\int_0^1T_s\dd s\) used in Section~\ref{subsec:flow-choices}.

The dimension reduction from signature to logsignature can be substantial. For a
\(q\)-dimensional path, the raw truncated signature has
\(1+\sum_{k=1}^m q^k\) coordinates including the constant coordinate, whereas the
nonconstant truncated logsignature has dimension
\[
\sum_{k=1}^m \ell_q(k),
\qquad
\ell_q(k):=\frac1k\sum_{r\mid k}\mu(r)q^{k/r},
\]
where \(\mu\) is the M\"obius function. Thus, for the time-augmented clock
\((q=2)\) and truncation order \(m=3\), the raw nonconstant signature has
dimension \(2+4+8=14\), while the logsignature has dimension \(2+1+2=5\). This
compactness is the main reason that the experiments use truncated logsignature
features for \(\eta_T\) when conditioning the neural velocity field.

\subsection{Signature-linear schedules}
\label{appendix:signature-linear-schedules}

This subsection gives the schedule details that are summarized in
Section~\ref{subsec:flow-choices}. The compactness assumption in
Proposition~\ref{prop:signature-linear-universality} below is the usual setting for
uniform approximation. For heavy-tailed clocks, which need not live on a compact
path set globally, the proposition should be read either on compact truncations
of the clock space, on high-probability compact subsets, or after applying the
feature normalizations used in training. The time augmentation is important:
it records the order in which clock mass is accumulated, so path functionals
such as \(\int_0^1T_s\dd s\) can be represented by low-order signature
information.

\begin{proposition}[Signature-linear approximation of clock functionals]
\label{prop:signature-linear-universality}
Let \(\mathcal K\) be a compact set of time-augmented, piecewise linear clock
paths \(\mathbf T=(s,T_s)_{s\in[0,1]}\) on which the signature separates paths
up to tree-like equivalence. Let
\(F:[0,1]\times\mathcal K\to\RR\) be continuous. Then for every
\(\varepsilon>0\), there exist a truncation level \(m\) and a continuous
coefficient curve \(c:[0,1]\to\RR^{D_m}\) such that
\begin{equation}
\sup_{t\in[0,1]}\sup_{T\in\mathcal K}
\left|F(t,T)-\langle c(t),\phi_m(T)\rangle\right|
<\varepsilon .
\label{eq:signature-universality-functional}
\end{equation}
In particular, clock-dependent schedule functionals such as
\(T\mapsto\alpha_t(T)\) and \(T\mapsto\beta_t(T)\) can be represented, up to
arbitrary uniform accuracy on compact clock families, by linear functionals of a
sufficiently high-order truncated signature:
\begin{equation}
\alpha_t(T)\approx \langle a(t),\phi_m(T)\rangle,
\qquad
\beta_t(T)\approx \langle b(t),\phi_m(T)\rangle .
\label{eq:signature-universality-schedule}
\end{equation}
\end{proposition}

The proposition motivates representing clock-dependent interpolation
coefficients by signature-linear functions. A direct form is
\begin{equation}
\alpha_t^{T,m}=\langle a(t),\phi_m(T)\rangle,
\qquad
\beta_t^{T,m}=\langle b(t),\phi_m(T)\rangle,
\label{eq:signature-linear-schedules}
\end{equation}
where \(a(t),b(t)\in\RR^{D_m}\) are coefficient curves. If these curves are
absolutely continuous in \(t\), then for a fixed realized clock path the
physical-time derivatives are explicit:
\begin{equation}
\frac{\dd\alpha_t^{T,m}}{\dd t}
=
\left\langle \frac{\dd a(t)}{\dd t},\phi_m(T)\right\rangle,
\qquad
\frac{\dd\beta_t^{T,m}}{\dd t}
=
\left\langle \frac{\dd b(t)}{\dd t},\phi_m(T)\right\rangle.
\label{eq:signature-schedule-derivatives}
\end{equation}
Thus the flow-matching target differentiates the scalar-time coefficient curves,
not the clock path itself.

The direct form \eqref{eq:signature-linear-schedules} does not automatically
enforce the endpoint constraints \(
\alpha_0^T=0,
\
\beta_0^T=1,
\
\alpha_1^T=1,
\
\beta_1^T=0,
\
\beta_t^T>0\ \text{for }t\in[0,1)
\).
A boundary-corrected near-linear schedule is obtained as follows. Let
\(\lambda:[0,1]\to\RR\) be absolutely continuous with \(\lambda_0=0\), and set
\begin{equation}
A_t(T):=\langle a(t),\eta_T\rangle,
\qquad
B_t(T):=\langle b(t),\eta_T\rangle.
\label{eq:signature-schedule-ab}
\end{equation}
Then define
\begin{equation}
\alpha_t^T
=
t+\lambda_t(1-t)A_t(T),
\qquad
\beta_t^T
=
1-t+\lambda_t(1-t)B_t(T)
=
(1-t)\bigl(1+\lambda_tB_t(T)\bigr).
\label{eq:near-linear-independent-schedule}
\end{equation}
The endpoint constraints are automatic:
\begin{equation}
\alpha_0^T=0,\quad
\alpha_1^T=1,\quad
\beta_0^T=1,\quad
\beta_1^T=0.
\label{eq:near-linear-endpoints}
\end{equation}
The path is admissible whenever
\begin{equation}
1+\lambda_tB_t(T)>0,
\qquad t\in[0,1).
\label{eq:near-linear-positivity}
\end{equation}
Writing
\[
\frac{\dd A_t(T)}{\dd t}
=
\left\langle\frac{\dd a(t)}{\dd t},\eta_T\right\rangle,
\qquad
\frac{\dd B_t(T)}{\dd t}
=
\left\langle\frac{\dd b(t)}{\dd t},\eta_T\right\rangle,
\]
the required derivatives are
\begin{align}
\frac{\dd\alpha_t^T}{\dd t}
&=
1+\left(\frac{\dd\lambda_t}{\dd t}(1-t)-\lambda_t\right)A_t(T)
+\lambda_t(1-t)\frac{\dd A_t(T)}{\dd t},
\label{eq:near-linear-alpha-derivative}\\
\frac{\dd\beta_t^T}{\dd t}
&=
-\bigl(1+\lambda_tB_t(T)\bigr)
+(1-t)\left(\frac{\dd\lambda_t}{\dd t}B_t(T)+\lambda_t\frac{\dd B_t(T)}{\dd t}\right),
\label{eq:near-linear-beta-derivative}\\
\frac{\frac{\dd\beta_t^T}{\dd t}}{\beta_t^T}
&=
-\frac{1}{1-t}
+
\frac{\frac{\dd\lambda_t}{\dd t}B_t(T)+\lambda_t\frac{\dd B_t(T)}{\dd t}}
{1+\lambda_tB_t(T)}.
\label{eq:near-linear-beta-log-derivative}
\end{align}
The singularity at \(t=1\) is the same terminal singularity present in the
straight-line flow-matching target and is irrelevant for the time-integrated
objective, where \(t=1\) is not sampled.

A symmetric special case uses a single clock-dependent time warp. Let
\begin{equation}
C_t(T):=\langle c(t),\eta_T\rangle,
\qquad
\gamma_t^T:=t+\lambda_t(1-t)C_t(T),
\label{eq:near-linear-time-warp}
\end{equation}
and set
\begin{equation}
\alpha_t^T=\gamma_t^T,
\qquad
\beta_t^T=1-\gamma_t^T=(1-t)\bigl(1-\lambda_tC_t(T)\bigr).
\label{eq:near-linear-warp-schedule}
\end{equation}
This preserves the straight-line form
\(X_t^T=\gamma_t^TX_1^T+(1-\gamma_t^T)X_0^T\), while allowing the realized clock path
to accelerate or decelerate the interpolation. The endpoint-conditioned target
simplifies to
\begin{equation}
v_t^T(x\mid x_1)
=
\frac{\frac{\dd\gamma_t^T}{\dd t}}{1-\gamma_t^T}(x_1-x),
\qquad t\in[0,1).
\label{eq:near-linear-warp-target}
\end{equation}

\subsection{VE-type signature-linear schedule}
\label{appendix:ve-signature-schedule}

For the clock-aware trajectory ablation, we use a variance-exploding
(VE)-type affine path written in the signature-linear form
\eqref{eq:signature-linear-schedules}. This is a deterministic affine flow path
inspired by VE noise schedules, not an additional stochastic SDE. Since all
clocks in this experiment satisfy \(T_1>0\), define the normalized clock
\(\bar T_s:=T_s/T_1\) and the normalized time-augmented path
\(\bar{\mathbf T}_s=(s,\bar T_s)\). Let \(\phi_m(\bar T)\) denote the raw
truncated signature of \(\bar{\mathbf T}\). For \(m\ge2\), let
\(e_{\mathtt c\mathtt t}\) and \(e_{\mathtt t\mathtt c}\) be the unit vectors
selecting the level-two signature coordinates with words
\((\mathtt c,\mathtt t)\) and \((\mathtt t,\mathtt c)\), where \(\mathtt t\)
denotes physical time and \(\mathtt c\) denotes the normalized clock coordinate.
Then
\begin{align}
\left\langle e_{\mathtt c\mathtt t},\phi_m(\bar T)\right\rangle
&=
\int_{0<u<v<1}\dd \bar T_u\,\dd v
=
\int_0^1 \bar T_u\,\dd u,
\label{eq:clock-time-area-ct}\\
\left\langle e_{\mathtt t\mathtt c},\phi_m(\bar T)\right\rangle
&=
\int_{0<u<v<1}\dd u\,\dd \bar T_v
=
1-\int_0^1 \bar T_v\,\dd v .
\label{eq:clock-time-area-tc}
\end{align}
Thus the signed clock-shape score
\begin{equation}
R_m(T)
:=
\left\langle r_m,\phi_m(\bar T)\right\rangle,
\qquad
r_m:=e_{\mathtt c\mathtt t}-e_{\mathtt t\mathtt c},
\label{eq:signature-clock-shape-score}
\end{equation}
satisfies
\begin{equation}
R_m(T)
=
2\int_0^1 \bar T_s\,\dd s-1.
\label{eq:signature-clock-shape-score-area}
\end{equation}
It is zero for a linear normalized clock, positive when clock mass accumulates
early, negative when it accumulates late, and bounded by \([-1,1]\) for
nondecreasing normalized clocks.

We use the normalized linear VE decay \(\sigma_{\rm VE}(t)=1-t\) and introduce
independent signature-linear clock-shape corrections:
\begin{equation}
\alpha_t^{T,m}
=
t+\lambda_\alpha t(1-t)R_m(T),
\qquad
\beta_t^{T,m}
=
(1-t)\bigl(1+\lambda_\beta tR_m(T)\bigr).
\label{eq:ve-signature-linear-schedule}
\end{equation}
Equivalently, with \(e_{\emptyset}\) denoting the constant signature coordinate,
\begin{align}
a_{\rm VE}(t)
&=
 t e_{\emptyset}+\lambda_\alpha t(1-t)r_m,
\label{eq:ve-signature-linear-a}\\
b_{\rm VE}(t)
&=
 (1-t)e_{\emptyset}+\lambda_\beta t(1-t)r_m,
\label{eq:ve-signature-linear-b}
\end{align}
so that
\begin{equation}
\alpha_t^{T,m}=\left\langle a_{\rm VE}(t),\phi_m(\bar T)\right\rangle,
\qquad
\beta_t^{T,m}=\left\langle b_{\rm VE}(t),\phi_m(\bar T)\right\rangle .
\label{eq:ve-schedule-signature-linear-form}
\end{equation}
In the experiments we use \(\lambda_\alpha=\lambda_\beta=1/4\). More generally,
since \(|R_m(T)|\le1\), any \(|\lambda_\beta|<1\) gives
\(\beta_t^{T,m}\ge (1-t)(1-|\lambda_\beta|)>0\) for \(t<1\), so the affine path
is admissible.

Writing \(R_m=R_m(T)\), the derivatives are
\begin{equation}
\frac{\dd\alpha_t^{T,m}}{\dd t}
=
1+\lambda_\alpha(1-2t)R_m,
\qquad
\frac{\dd\beta_t^{T,m}}{\dd t}
=
-1+\lambda_\beta(1-2t)R_m,
\label{eq:ve-signature-linear-derivatives}
\end{equation}
and hence
\begin{equation}
L_t^{T,m}
:=
\frac{\frac{\dd\beta_t^{T,m}}{\dd t}}{\beta_t^{T,m}}
=
\frac{-1+\lambda_\beta(1-2t)R_m}
{(1-t)\bigl(1+\lambda_\beta tR_m\bigr)}.
\label{eq:ve-signature-linear-beta-log-derivative}
\end{equation}
Substituting these quantities into \eqref{eq:clock-affine-target} gives
\begin{align}
v_t^{T,m}(x\mid x_1)
&=
L_t^{T,m}x
+
\left(
\frac{\dd\alpha_t^{T,m}}{\dd t}-L_t^{T,m}\alpha_t^{T,m}
\right)x_1
\nonumber\\
&=
L_t^{T,m}x
+
\left[
1+\lambda_\alpha(1-2t)R_m
-
L_t^{T,m}\bigl(t+\lambda_\alpha t(1-t)R_m\bigr)
\right]x_1,
\qquad t\in[0,1).
\label{eq:ve-signature-linear-target}
\end{align}
When \(\lambda_\alpha=\lambda_\beta=0\), this reduces to the linear VE path
\(\alpha_t=t\), \(\beta_t=1-t\). The one-warp path is recovered only in the
special symmetric case \(\lambda_\beta=-\lambda_\alpha\); the ablation above
uses the more general affine form
\(X_t^T=\alpha_t^{T,m}X_1^T+\beta_t^{T,m}X_0^T\).


\section{Experimental Setup}
\label{appendix:experimental-setup}
\label{appendix:more-experimental-results}
\label{appendix:experimental-setting}

This section documents the full configuration shared across all reported
HTFM and baseline runs. Unless explicitly noted, any HTFM clock variant
(Gaussian, $\alpha$-stable, Student-$t$) uses the same backbone and
optimizer for a given dataset; the only differences are the random-clock
law and the optional path-feature conditioning. All experiments are
implemented in PyTorch.

\subsection{Dataset}
\label{appendix:setup-dataset}

We use three benchmarks at increasing scale.

\paragraph{2D imbalanced $\alpha$-stable mixture.}
\label{appendix:toy-setup}
A $9$-component isotropic $\alpha$-stable mixture in $\mathbb R^2$.
Components sit on a uniformly spaced
$3{\times}3$ grid with spacing $2.0$, each component is a $2$-D isotropic
$\alpha$-stable variable with scale $0.1$, and the mixture weights are:\\
$\{0.18, 0.08, 0.14, 0.06, 0.09, 0.10, 0.16, 0.07, 0.12\}$ (range
$0.06$--$0.18$). We sweep the per-component stability index
$\alpha_{D}\in\{1.5, 1.8\}$ and draw $32{,}000$ training
samples per setting; an additional $24{,}000$-sample reference set is held
out for evaluation.

\paragraph{CIFAR10-LT.}
The long-tailed CIFAR-10 split~\citep{cao2019learning} of the original
CIFAR-10 train set~\citep{krizhevsky2009learning}. With imbalance ratio
$\rho=0.01$, head count $5000$, and $10$ classes, the per-class counts
follow a geometric schedule and total $\approx 12{,}406$ training images
at $32{\times}32$. A deterministic subset is precomputed and reused across all runs to keep the baselines
identical at the dataset level. Random horizontal flips are the only
augmentation.

\paragraph{HRRR VIL.}
Vertically Integrated Liquid (VIL) hourly analyses from the
High-Resolution Rapid Refresh weather system~\citep{benjamin2016north},
cropped to a $128{\times}128$ window over the Central US following the
protocol of~\cite{pandey2024heavy}. Train/test split mirrors theirs:
years $\{2019, 2020\}$ for training ($\approx 17{,}400$ valid hourly
fields after dropping NaN slots) and year $\{2021\}$ for test
($\approx 8{,}700$). Per-channel $z$-score statistics are computed on the
training years and applied to both splits.

\subsection{Baselines}
\label{appendix:setup-baselines}

We compare HTFM against four heavy-tailed generative baselines, each run
under the same backbone and training budget as HTFM (see
\Cref{appendix:setup-denoiser,appendix:setup-training}):
\begin{itemize}
\item \textbf{LIM}~\citep{yoon2023score,popovimproved}: L\'evy-It\^o stochastic
  differential equation with $\alpha$-stable driving noise; trained with a
  fractional-score regression loss and sampled with the model's
  $\alpha$-stable reverse SDE solver at reverse-step budget $1000$.
\item \textbf{DLPM}~\citep{shariatianheavy}: discrete-time $\alpha$-stable
  analogue of DDPM; uses $4000$ training-time reverse steps with the
  Gaussian scale-mixture augmentation, sampling via its native stochastic predictor.
\item \textbf{DLIM}~\citep{shariatianheavy}: deterministic $\alpha$-stable
  sampler sharing the DLPM predictor backbone; We set $\eta{=}0$ for deterministic generation.
\item \textbf{t-Flow}~\citep{pandey2024heavy}: heavy-tailed flow with a
  finite-variance Student-$t$ source; reported only on the Student-$t$
  parameter range $p_\nu>2$ where its construction is well-defined.
\end{itemize}

\subsection{Evaluation}
\label{appendix:setup-evaluation}

Each benchmark uses the standard tail-aware metric for its scale.

\paragraph{2D mixture.}
We use the precision--recall $f_1^{\mathrm{pr}}$ score
of~\cite{shariatianheavy}. Let
$\mathcal X_{\mathrm{ref}}=\{x_i\}_{i=1}^{N}$ and
$\mathcal X_{\mathrm{gen}}=\{y_j\}_{j=1}^{N}$ be the reference and
generated sample sets ($N=24{,}000$ here). For a set
$\mathcal X$ and a point $z$, write
$r_k(z;\mathcal X)$ for the distance from $z$ to its $k$-th nearest
neighbor in $\mathcal X$, and define the $k$-NN manifold
$$
\mathcal M_k(\mathcal X) \;=\; \bigcup_{z\in\mathcal X}
B\bigl(z,\,r_k(z;\mathcal X)\bigr),
$$
i.e.\ the union of $\ell_2$-balls around each sample with radius given
by its own $k$-th nearest-neighbor distance. The
\emph{precision} and \emph{recall} relative to this support estimate
are
$$
\mathrm{P} \;=\; \frac{1}{|\mathcal X_{\mathrm{gen}}|}
\sum_{y\in\mathcal X_{\mathrm{gen}}}
\mathbf 1\!\bigl[y\in\mathcal M_k(\mathcal X_{\mathrm{ref}})\bigr],
\qquad
\mathrm{R} \;=\; \frac{1}{|\mathcal X_{\mathrm{ref}}|}
\sum_{x\in\mathcal X_{\mathrm{ref}}}
\mathbf 1\!\bigl[x\in\mathcal M_k(\mathcal X_{\mathrm{gen}})\bigr],
$$
and the score is their harmonic mean
$$
f_1^{\mathrm{pr}} \;=\; \frac{2\,\mathrm{P}\,\mathrm{R}}{\mathrm{P}+\mathrm{R}}\,.
$$
Higher is better. We use $k=10$ throughout.

\paragraph{CIFAR10-LT.}
We report Fr\'echet Inception Distance
(FID)~\citep{heusel2017gans}, the standard image-generation
metric.
Each $32{\times}32$ image (real and generated) is passed through a
pre-trained Inception-V3 network and embedded as a $2048$-dimensional
feature vector at the final average-pooling layer. Letting
$(\mu_{\mathrm{data}},\Sigma_{\mathrm{data}})$ and
$(\mu_{\mathrm{gen}},\Sigma_{\mathrm{gen}})$ be the empirical mean and
covariance of the two feature sets, FID is the squared
$2$-Wasserstein distance between the two fitted Gaussians,
\begin{equation}
\mathrm{FID} \;=\;
\lVert\mu_{\mathrm{data}}-\mu_{\mathrm{gen}}\rVert_2^{2}
\;+\;
\mathrm{Tr}\!\bigl(
  \Sigma_{\mathrm{data}}+\Sigma_{\mathrm{gen}}
  -2(\Sigma_{\mathrm{data}}\Sigma_{\mathrm{gen}})^{1/2}
\bigr).
\end{equation}
The reference set is the full long-tail train split
($12{,}406$ images at $32{\times}32$); the generated set is
$12{,}406$ images sampled from the EMA-averaged model copy at the
listed NFE budget. Lower is better.

\paragraph{HRRR VIL.}
Following~\cite{pandey2024heavy}, we report three one-dimensional
tail statistics aligned to the operational interest in HRRR extremes.
We generate $20{,}000$ samples from each model, flatten them
together with the $\approx 8{,}760$ test fields into 1-D
intensity samples, and compute the metrics below. Lower is better
for all three.

\textbf{Kurtosis Ratio (KR).} Sample kurtosis is the (normalized)
fourth-order moment, characterizing how heavy the distribution's tails
are relative to the central mass. Let $k_{\mathrm{data}}$ and
$k_{\mathrm{sim}}$ denote the empirical kurtosis of the reference and
generated samples. The kurtosis ratio is
\begin{equation}
\mathrm{KR} \;=\; \Bigl| 1 - \tfrac{k_{\mathrm{sim}}}{k_{\mathrm{data}}}\Bigr|.
\end{equation}

\textbf{Skewness Ratio (SR).} Sample skewness is the (normalized)
third-order moment, characterizing the asymmetry of a tailed
distribution. With $s_{\mathrm{data}}, s_{\mathrm{sim}}$ the empirical
skewness on reference and generated,
\begin{equation}
\mathrm{SR} \;=\; \Bigl| 1 - \tfrac{s_{\mathrm{sim}}}{s_{\mathrm{data}}}\Bigr|.
\end{equation}

\textbf{Tail Kolmogorov--Smirnov (KS).} The two-sample
KS statistic~\citep{massey1951kolmogorov} measures the sup-norm
distance between two empirical CDFs. For heavy-tailed distributions
we evaluate it on the tails to focus the metric on the rare-extreme
regime that both KR and SR weight only indirectly. We retain
generated and reference samples lying above the $99.9$-th percentile
(right tail) or below the $0.1$-th percentile (left tail), compute
the two-sample KS statistic per side, and average over the two sides
for an overall score. For the VIL channel the underlying distribution
has no left tail, so we report the right-tail KS statistic only.

\subsection{Network Architecture}
\label{appendix:setup-denoiser}

\paragraph{2D mixture.}
A small MLP with $4$ residual blocks of width $64$, SiLU activations,
GroupNorm, and a learnable time embedding of size $32$. The same backbone
is shared across HTFM and the FM baselines.

\paragraph{CIFAR10-LT.}
The DDPM-style U-Net of~\cite{ho2020denoising}: $128$ base channels,
channel multipliers $(1,2,2,2)$, two residual blocks per scale, $4$
attention heads, self-attention at resolution $16$, dropout $0$. The
\emph{same} U-Net is used by HTFM, LIM, DLPM, DLIM, and t-Flow so that
any FID difference reflects the noising/clock formulation rather than
backbone capacity. For HTFM, the U-Net is additionally equipped with the
path-feature conditioning module described below.

\paragraph{HRRR VIL.}
The NCSN++ backbone of~\cite{song2020score} on $128{\times}128$
single-channel inputs, with base channels $32$, channel multipliers
$(1,2,2,4,4)$, four residual blocks per scale, $4$ attention heads,
self-attention at resolution $16$, FIR $(1,3,3,1)$ up/down-sampling, swish
nonlinearities, GroupNorm, and BigGAN-style residual blocks. Sigma range
$[0.01, 50]$ on $1000$ noise scales when used with score-based baselines.

\paragraph{Path-feature conditioning (HTFM-specific).}
HTFM's velocity network conditions on a clock-derived \emph{path feature}
$\phi$ that summarizes the realized clock trajectory of each sample. The
feature source depends on the random clock:
\begin{itemize}[leftmargin=1.2em,topsep=2pt,itemsep=1pt]
\item \textbf{Student-$t$ clock} $T_t{=}tV_\nu$,
  $V_\nu\sim\mathrm{InvGamma}(\nu/2,\nu/2)$. The path is fully determined
  by the single scalar $V_\nu$, so we take $\phi{=}\log V_\nu\in\mathbb R$
  ($d_\phi{=}1$); no signature is needed.
\item \textbf{$\alpha$-stable subordinator clock}. We discretize the
  time-augmented path $\{(t_i, T_{t_i})\}_{i=0}^{N}$ on a fixed
  $N{=}200$ grid, standardize each of the two channels along time, and take its truncated
  logsignature of order $m$ (default $m{=}1$, i.e.\ $\phi\in\mathbb R^{2}$
  collecting the increments of the two channels). The construction
  matches \Cref{subsec:clock-signature-schedules}.
\item \textbf{Deterministic clock} $T_t{=}t$ (HTFM-G, $p{=}\infty$). No
  path feature; the conditioning module is omitted entirely and the
  network reduces to a standard time-conditioned velocity field.
\end{itemize}

The raw feature $\phi\in\mathbb R^{d_\phi}$ is mapped through a shared
two-layer MLP
$\phi_{\rm path}\colon\mathbb R^{d_\phi}\to\mathbb R^{128}$ with SiLU
activations, dropout $0$, and an internal hidden width of $128$. The
$128$-dimensional path embedding is injected into the velocity network at
\emph{every} residual block by concatenation with the time embedding. In
the original block the time embedding $t_{\rm emb}\in\mathbb R^{d_t}$ is
mapped to a per-channel scale/shift via a linear $W_t\in\mathbb
R^{C_{\rm out}\times d_t}$; with path conditioning this is replaced by a
wider linear $W_{t,\phi}\in\mathbb R^{C_{\rm out}\times (d_t+128)}$ acting
on the concatenation $[t_{\rm emb}\,\|\,\phi_{\rm path}]$. The same
concatenation-at-every-block scheme is used identically by the DDPM U-Net
(CIFAR10-LT), the NCSN++ U-Net (HRRR VIL), and the residual-MLP backbone
(2D mixture); this keeps the comparison against LIM/DLPM/DLIM/t-Flow
clean, since their backbones have no analogous trajectory-summary input
and therefore see only the time embedding.

For completeness, our implementation also supports a zero-initialized
FiLM modulation interface that produces a per-channel
$(\mathrm{scale},\mathrm{shift})$ from $\phi_{\rm path}$ and applies it
at the middle block. All HTFM rows reported in the paper use only the
concatenation interface; FiLM is left disabled.

\subsection{Training}
\label{appendix:setup-training}

All runs use AdamW with default $\beta$'s. Per-dataset hyperparameters
are listed below.

\paragraph{2D mixture.}
Learning rate $5{\times}10^{-3}$, batch size $1024$, EMA decay $0.99$,
total number of epochs $100$.

\paragraph{CIFAR10-LT.}
Learning rate $2{\times}10^{-4}$ with $500$ warmup steps and a step-LR
schedule (decay factor $0.99$ every $1000$ steps); batch size $100$;
EMA decay $0.9999$ used for evaluation; gradient clipping at $\ell_2$ norm
$1.0$. Models are trained for $2000$ epochs. 

\paragraph{HRRR VIL.}
Learning rate $10^{-4}$ with $1000$ warmup steps and no LR decay; weight
decay $0$; batch size $128$; EMA decay $0.9999$; gradient clipping at
$5.0$. Models are trained for $1500$ epochs.
\paragraph{HTFM-specific knobs.}
For HTFM the random clock is one of: deterministic $T_t=t$ (Gaussian
limit, HTFM-G), an $\alpha$-stable subordinator with decay degree
$p_\alpha\in\{1.6,1.7,1.8,1.9\}$ on CIFAR10-LT and
$p_\alpha\in\{1.6,1.7,1.8\}$ on HRRR, or a Student-$t$ random-slope clock
$T_t=tV_\nu$, $V_\nu\sim\mathrm{InvGamma}(\nu/2,\nu/2)$, with decay
degree $p_\nu\in\{1.7,2.0,3.0\}$ on CIFAR10-LT and HRRR (extended to
$p_\nu=1.5$ on the 2D mixture where a degenerate variance is acceptable).
The training loss is the clock-conditioned $L^2$ flow-matching objective
with velocity prediction.


\subsection{Sampling}
\label{appendix:setup-sampling}

All HTFM and FM-based variants integrate the learned velocity field with
a fixed-step ODE solver; LIM, DLPM, and DLIM use their respective
(stochastic) reverse processes.

\paragraph{2D mixture.}
Deterministic Euler ODE solver;
$24{,}000$ samples generated per setting and compared to a held-out
$24{,}000$-sample reference set.

\paragraph{CIFAR10-LT.}
Deterministic Euler ODE solver for HTFM and the baselines. DLPM uses its native predictor; DLIM uses the deterministic ($\eta{=}0$) variant. Higher-order solvers are reported only in the
solver ablation (Appendix~\ref{appendix:solver-ablation}). Unless otherwise stated, all FID results use the final EMA checkpoint after the fixed training budget.

\paragraph{HRRR VIL.}
Heun's second-order ODE solver at $25$ reverse steps (NFE $=50$ per
sample), following~\cite{pandey2024heavy}. Tail statistics are evaluated
on $20{,}000$ generated samples.

\subsection{Compute Resources}
\label{app:compute}

All experiments were conducted on a mix of NVIDIA L40S (48GB) and A100 (80GB) GPUs. CIFAR10-LT runs use up to 8 L40S GPUs in parallel, while HRRR VIL runs use A100 80G GPUs to accommodate the larger $128\times128$ single-channel inputs and the NCSN++ backbone. The total training budget is fixed by the per-dataset epoch counts reported in Appendix~\ref{appendix:setup-training} (100 epochs for the 2D mixture, 2000 epochs for CIFAR10-LT, and 1500 epochs for HRRR VIL). Sampling cost is dominated by the listed NFE budgets and the ODE/SDE solver choices in Appendix~\ref{appendix:setup-sampling}.

\section{Ablation Studies}
\label{appendix:ablation-studies}

Unless otherwise stated, ablations are conducted on CIFAR10-LT because it
exposes both global generation quality and tail-class coverage while remaining
more economical than the HRRR sweep. The same protocol can be repeated on 2D
toy distributions for visualization and on HRRR for the final robustness check.

\subsection{Effect of Flow Type}
\label{appendix:flow-type-ablation}

\begin{wraptable}{r}{0.5\linewidth}
\vspace{-8pt}
\setlength{\abovecaptionskip}{5pt}
\setlength{\belowcaptionskip}{0pt}
\centering

\setlength{\tabcolsep}{4.2pt}
\renewcommand{\arraystretch}{1.08}
\resizebox{\linewidth}{!}{%
\begin{tabular}{lccc}
\toprule
Flow type & Level 0 & Level 1 & Level 2 \\
\midrule
Straight line 
& 16.65 & 16.10 & 16.46 \\
Signature-linear VE
& 16.43 & 16.19 & 16.71 \\
\bottomrule
\end{tabular}%
}
\caption{Flow-type ablation on CIFAR10-LT under HTFM-$\alpha$ at
$p_\alpha=1.7$. We compare two clock-conditioned interpolation schedules in FID performance:
the clock-blind straight line and a signature-linear VE schedule whose
mixing weights depend on log-signature features of the realized clock path.
Columns vary the signature order provided to the velocity network. Each
cell reports the FID at sampling NFE$=100$ using the Euler ODE solver.}
\label{tab:flow-type-ablation}
\vspace{-8pt}
\end{wraptable}

We next ablate the interpolation \emph{schedule} itself, holding the random
clock and the source covariance fixed. We compare two flow types implemented
in our codebase: (i) the standard \emph{straight-line} interpolation
$X_t^T=tX_1+(1-t)X_0^T$, which ignores the realized clock and moves linearly
in physical time $t$; and (ii) a \emph{signature-linear VE} schedule whose
interpolation coefficients are convex combinations parameterized by truncated
logsignature features of the clock path, with mixing weights
$\lambda_\alpha,\lambda_\beta$ controlling how aggressively the schedule
deviates from straight line. Comparing these two schedules under a matched
signature feature isolates the value of clock-aware versus clock-blind
interpolation. We run this ablation on the HTFM-$\alpha$ family at
$p_\alpha=1.7$ as a representative slice of the heavy-tailed regime.

\subsection{Effect of Number of Function Evaluations (NFE)}
\label{appendix:nfe-ablation}

We next vary the number of function evaluations (NFE) used by the sampler,
holding training fixed and changing only the sampling budget. This
ablation measures each method's quality--cost tradeoff:
\Cref{fig:nfe-ablation} plots FID against NFE for HTFM-$\alpha$, DLPM,
and DLIM at three decay degrees $p_\alpha\in\{1.6,1.7,1.8\}$.

\begin{figure}[t]
\centering
\includegraphics[width=\linewidth]{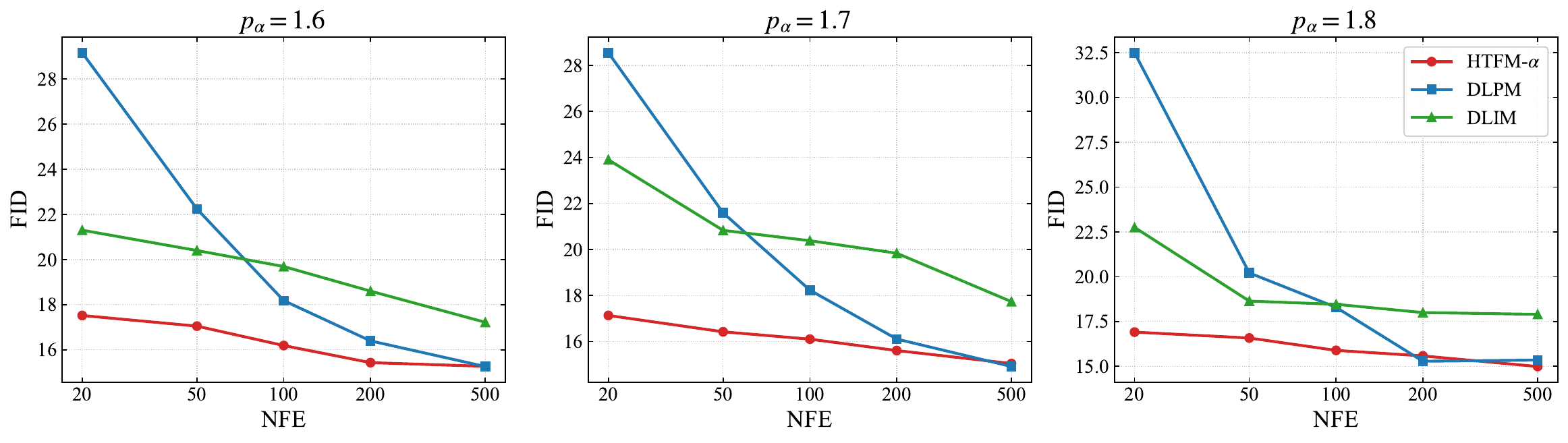}
\caption{NFE-vs-FID curves on CIFAR10-LT. One panel per decay degree
$p_\alpha\in\{1.6,1.7,1.8\}$; lines compare HTFM-$\alpha$ (ours) with the
DLPM and DLIM baselines. Lower FID is better.}
\label{fig:nfe-ablation}
\end{figure}

Three trends are visible across all three panels.
\textbf{(i) HTFM-$\alpha$ dominates at low NFE.} The red curve sits well
below both baselines from the smallest sampling budget, and the gap to
DLPM is largest precisely where compute is scarcest -- consistent with
HTFM inheriting the low-NFE sample efficiency of flow matching rather than
the slow burn-in of $\alpha$-stable diffusion.
\textbf{(ii) DLPM closes the gap only at high NFE.} The blue curve falls
steeply with more solver steps and crosses or matches HTFM-$\alpha$ near
the right edge of the budget range, indicating that the two methods become
broadly comparable at the asymptotic-quality end of the curve. The
crossover happens later for heavier tails (smaller $p_\alpha$), where the
sampler has more high-frequency content to integrate through.
\textbf{(iii) DLIM plateaus.} The green curve flattens early and stays
above both other methods for the rest of the range, so its deterministic
sampler does not scale further with extra NFE on this benchmark.
Together these curves give the same message as the main-text
CIFAR10-LT comparison: HTFM-$\alpha$ matches the asymptotic quality of
the strongest $\alpha$-stable diffusion baseline while strictly
dominating it under tight NFE budgets.

\subsection{Effect of ODE Solver}
\label{appendix:solver-ablation}

This ablation compares the integration scheme used during sampling, holding
the trained model fixed and varying only the solver and the matched NFE
budget. We compare four fixed-step ODE solvers implemented in our codebase:
explicit \emph{Euler} ($1$st order), \emph{Heun} and \emph{midpoint}
($2$nd-order single-step Runge--Kutta), and Adams--Bashforth \emph{AB2}
($2$nd-order multi-step).
Higher-order solvers reduce the per-step discretization error of the
velocity integration, but each step costs more network evaluations: a step
of Heun/midpoint/AB2 costs $2$ NFE. At a matched
total NFE budget the higher-order solvers therefore take fewer integration
steps, so the ablation captures the per-NFE quality--accuracy tradeoff
rather than a per-step comparison. We report HTFM-$\alpha$ at $p_\alpha=1.7$
as a representative slice.

\begin{table}[t]
\setlength{\abovecaptionskip}{7pt}
\setlength{\belowcaptionskip}{0pt}
\centering
\small
\setlength{\tabcolsep}{6pt}
\renewcommand{\arraystretch}{1.08}
\begin{tabular}{lccc}
\toprule
Solver & NFE $=20$ & NFE $=50$ & NFE $=100$ \\
\midrule
Euler ($1$st order)             & 20.80 & 17.92 & 16.31 \\
Heun ($2$nd order)              & 18.55 & 17.44 & 15.87 \\
Midpoint ($2$nd order)          & \textbf{17.82} & 16.51 & \textbf{15.81} \\
AB2 ($2$nd order, multi-step)   & 18.02 & \textbf{16.47} & 15.82 \\
\bottomrule
\end{tabular}%
\caption{ODE-solver ablation on CIFAR10-LT under HTFM-$\alpha$ at
$p_\alpha=1.7$. Each cell reports FID ($12{,}406$ generated vs $12{,}406$ reference, EMA, lower is better) at
the listed total NFE budget; for the higher-order solvers the number of
integration steps is $\lfloor \text{NFE}/\text{NFE-per-step}\rfloor$. \textbf{Bold} indicates per-column best.}
\label{tab:solver-ablation}
\end{table}

\section{Visualizations}
\label{appendix:visualizations}

This section collects qualitative visualizations that complement the
quantitative results in the main text.

\subsection{2D toy distributions: clock-degree sweep}
\label{appendix:viz-toy-sweep}

\Cref{fig:viz-toy-sweep} compares HTFM-$\alpha$ samples generated under
clocks of different decay degrees (rows: $p=\infty$,
$p_\alpha=1.8$, $p_\alpha=1.5$) against ground-truth samples (top
row) across five 2D toy targets (columns: \texttt{checkerboard},
\texttt{two\_moons}, \texttt{swiss\_roll}, \texttt{rings},
\texttt{olympic\_rings}). The same backbone, training schedule, and
NFE budget are used for every cell; only the random-clock law is
swept.

\begin{figure}[!h]
\centering
\includegraphics[width=1\linewidth]{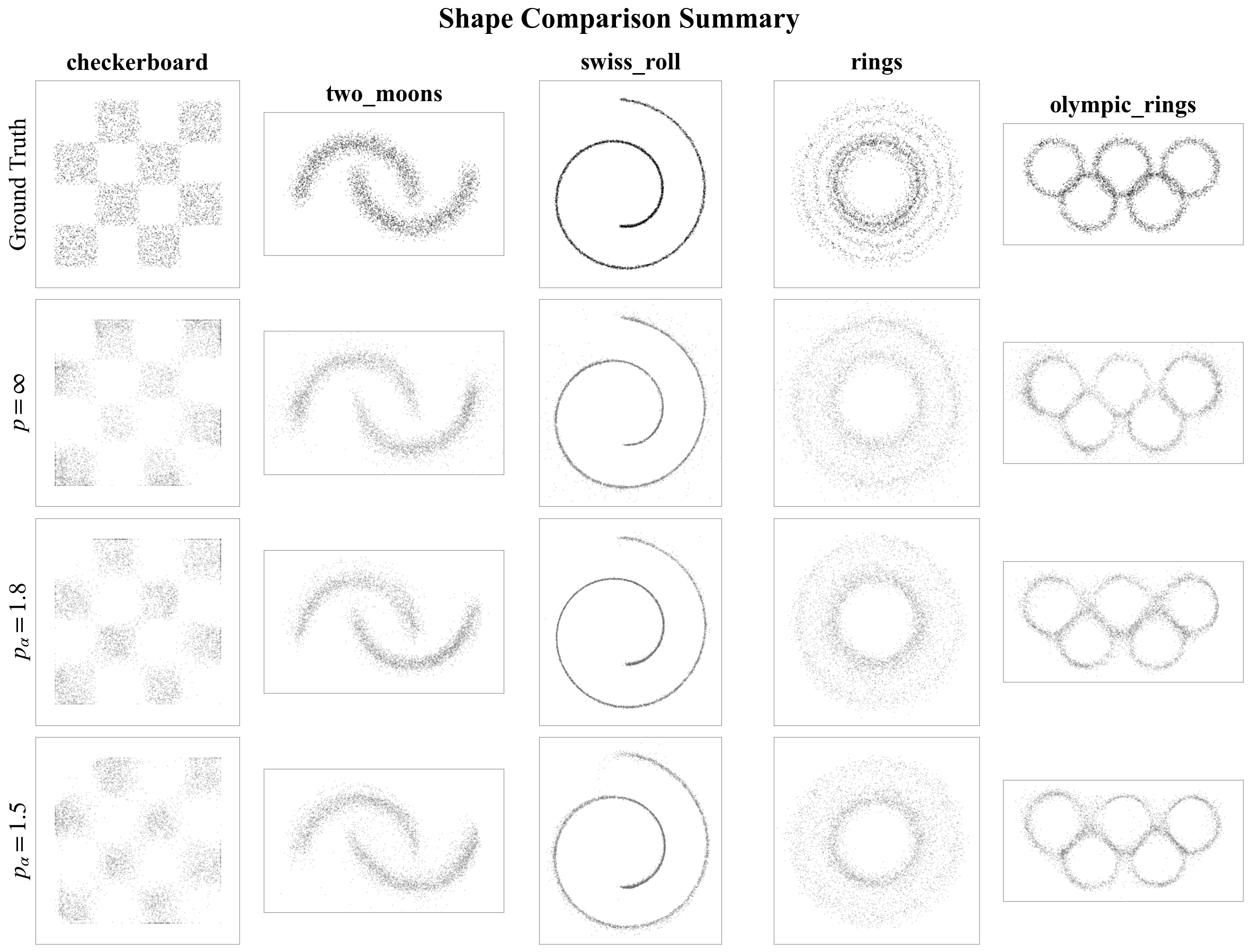}
\caption{HTFM-$\alpha$ samples on five 2D toy distributions
(\texttt{checkerboard}, \texttt{two\_moons}, \texttt{swiss\_roll},
\texttt{rings}, \texttt{olympic\_rings}) under decay exponents
$p\in\{\infty, 1.8, 1.5\}$ ($\alpha$ in the figure denotes the
$\alpha$-stable stability index of the random clock; the top row is
ground truth). Same backbone, training, and NFE across cells.}
\label{fig:viz-toy-sweep}
\end{figure}

\begin{figure}[!htbp]
\centering
\includegraphics[width=1\linewidth]{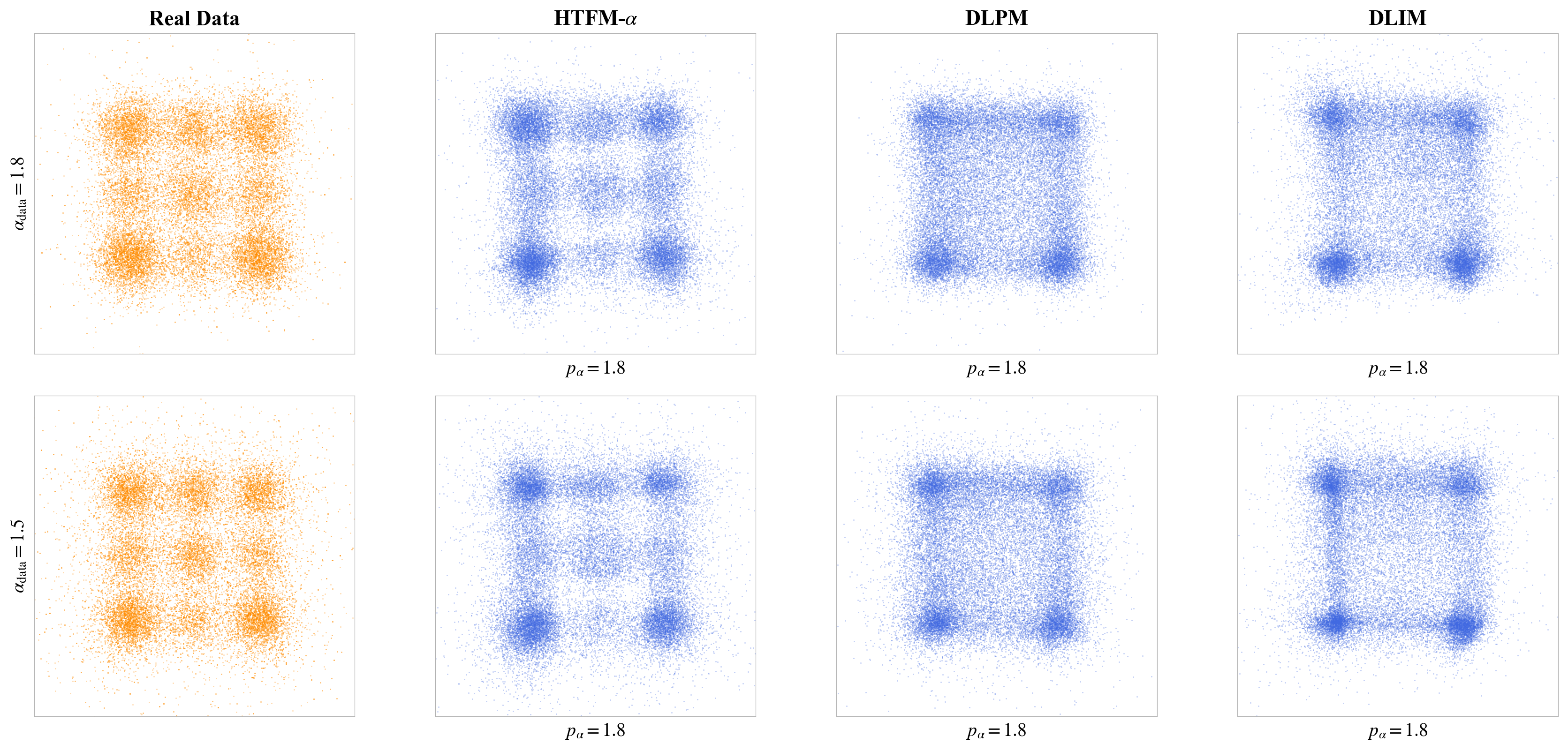}
\caption{Generated samples on the 2D imbalanced $\alpha$-stable
mixture at NFE${=}20$. Rows: $\alpha_{data}=1.8$ (top),
$\alpha_{data}=1.5$ (bottom). Columns: real data, HTFM,
DLPM, DLIM. $W_1$ distances to the real data are reported below each
generated panel.}
\label{fig:viz-sas-nfe20}
\end{figure}

\subsection{2D \texttt{S$\alpha$S} mixture: baseline comparison at NFE${=}20$}
\label{appendix:viz-sas-baselines}

\Cref{fig:viz-sas-nfe20} compares HTFM against DLPM and DLIM on the 2D
imbalanced $\alpha$-stable mixture at NFE${=}20$, on two stability
indices $\alpha_{data}\in\{1.5, 1.8\}$ (rows). Each panel
overlays generated samples (blue) on the corresponding real data
(orange in the leftmost column). HTFM reproduces the
imbalanced grid pattern more cleanly.




\subsection{CIFAR10-LT: generated samples}
\label{appendix:viz-cifar10lt}

\Cref{fig:viz-cifar10lt} shows generated $32{\times}32$ samples from
HTFM-$\alpha$ on CIFAR10-LT, organised so that each column collects
samples from a single CIFAR-10 class.

\begin{figure}[h]
\centering
\includegraphics[width=0.78\linewidth]{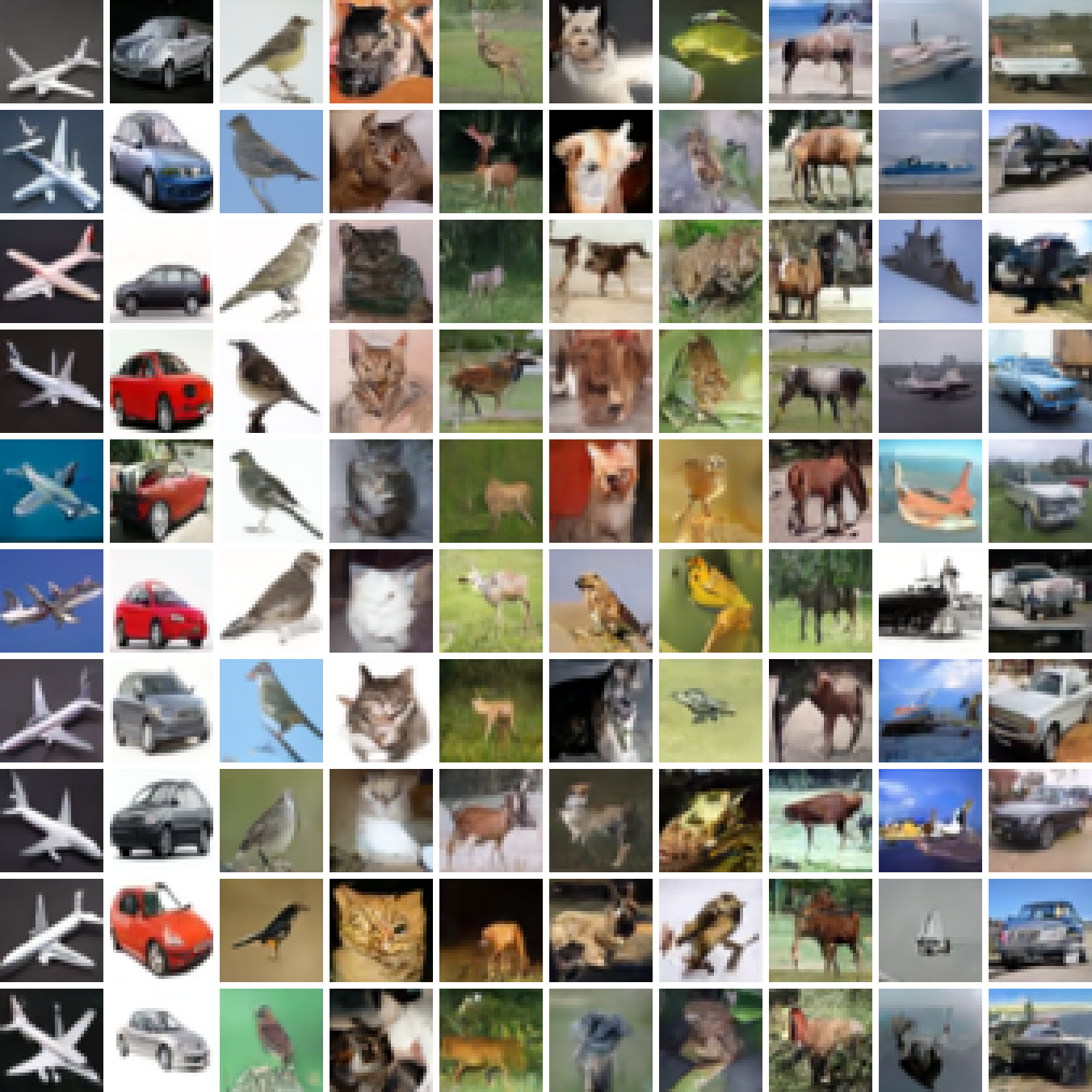}
\caption{CIFAR10-LT generated samples from HTFM-$\alpha$. Each column
corresponds to one CIFAR-10 class; from left to right: \emph{airplane,
automobile, bird, cat, deer, dog, frog, horse, ship, truck}. This is
also the CIFAR10-LT head-to-tail order, i.e.\ per-class training-sample
counts decrease from the leftmost to the rightmost column. Sample
fidelity broadly tracks the per-class training-set size: the
head classes \emph{airplane} and \emph{automobile} (leftmost two
columns) produce the sharpest, most coherent images, while
progressively rarer classes (rightward) exhibit more artefacts and
weaker class identity.}
\label{fig:viz-cifar10lt}
\end{figure}

\end{document}